\documentclass[11pt]{article} 
\usepackage{graphicx}
\usepackage{subfigure}
\usepackage{algorithm}
\usepackage{algorithmic}
\usepackage{amsthm}
\usepackage{amsmath}
\usepackage{amssymb}
\usepackage{bm}
\newcommand{\eps}{\epsilon}

\newcommand{\Var}{\mathbf{Var}}
\def\A{{\bf A}}
\def\a{{\bf a}}
\def\B{{\bf B}}

\def\c{{\bf c}}

\def\E{{\bf E}}

\def\I{{\bf I}}

\def\R{{\bf R}}
\def\X{{\bf X}}
\def\Y{{\bf Y}}
\def\p{{\bf p}}

\def\S{{\bf S}}

\def\T{{\bf T}}
\def\x{{\bf x}}
\def\y{{\bf y}}

\def\V{{\bf V}}

\def\0{{\bf 0}}
\def\1{{\bf 1}}

\newtheorem{lemma}{Lemma}
\newtheorem{definition}{Definition}
\newtheorem{theorem}{Theorem}

\newtheorem{proposition}{Proposition}
\newtheorem{cor}{Corollary}
\newtheorem{fact}{Fact}

\def \fourthm{{\beta}}
\def\balpha{\bm{\alpha}}
\usepackage[margin=1in]{geometry}
\title{Spectrum Estimation from Samples}

\author{
Weihao Kong \\\emph{whkong@stanford.edu} \and Gregory Valiant \\ \emph{valiant@stanford.edu}
}

\begin{document}

\maketitle

\vspace{-1cm}\begin{abstract}We consider the problem of approximating the set of eigenvalues of the covariance matrix of a multivariate distribution (equivalently, the problem of approximating the ``population spectrum''), given access to samples drawn from the distribution.  The eigenvalues of the covariance of a distribution contain basic information about the distribution, including the presence or lack of structure in the distribution, the effective dimensionality of the distribution, and the applicability of higher-level machine learning and multivariate statistical tools.   We consider this fundamental recovery problem in the regime where the number of samples is comparable, or even sublinear in the dimensionality of the distribution in question.   First, we propose a theoretically optimal and computationally efficient algorithm for recovering the moments of the eigenvalues of the population covariance matrix.  We then leverage this accurate moment recovery, via a Wasserstein distance argument, to show that the vector of eigenvalues can be accurately recovered.  We provide finite--sample bounds on the expected error of the recovered eigenvalues, which imply that our estimator is asymptotically consistent as the dimensionality of the distribution and sample size tend towards infinity, even in the sublinear sample regime where the ratio of the sample size to the dimensionality tends to zero.   In addition to our theoretical results, we show that our approach performs well in practice for a broad range of distributions and sample sizes.
\end{abstract}

\section{Introduction}
One of the most insightful properties of a multivariate distribution (or dataset) is the vector of eigenvalues of the covariance of the distribution or dataset.  This vector of eigenvalues---the ``spectrum''---contains important information about the structure and geometry of the distribution.   Indeed, the first step in understanding many high-dimensional distributions is to compute the eigenvalues of the covariance of the data, often with the aim of understanding whether there exist lower dimensional subspaces that accurately capture the majority of the structure of the high-dimensional distribution (for example, as a first step in performing Principal Component Analysis).      

Given independent samples drawn from a multivariate distribution over $\mathbb{R}^d$, when can this vector of eigenvalues of the (distribution/``population'') covariance be accurately computed?  In the regime in which the number of samples, $n$, is significantly larger than the dimension $d$, the empirical covariance matrix of the samples will be an accurate approximation of the true distribution covariance (assuming some modest moment bounds), and hence the empirical spectrum will accurately reflect the true population spectrum.   In the linear or sublinear regime in which $n$ is comparable to, or significantly smaller than $d$, the empirical covariance will be significantly different from the population covariance of the distribution.  Both the eigenvalues, and corresponding eigenvectors (principal components) of this empirical covariance matrix may be misleading.  (See Figure~\ref{fig1} for an illustration of this fact.)  The basic question we consider and answer affirmatively is: \emph{In this linear or sublinear sample regime in which the eigenvalues of the empirical covariance are misleading, is it possible to recover accurate estimates of the eigenvalues of the underlying population covariance?}  

\begin{figure}[h]
 \centering 
 \includegraphics[width=0.9\textwidth]{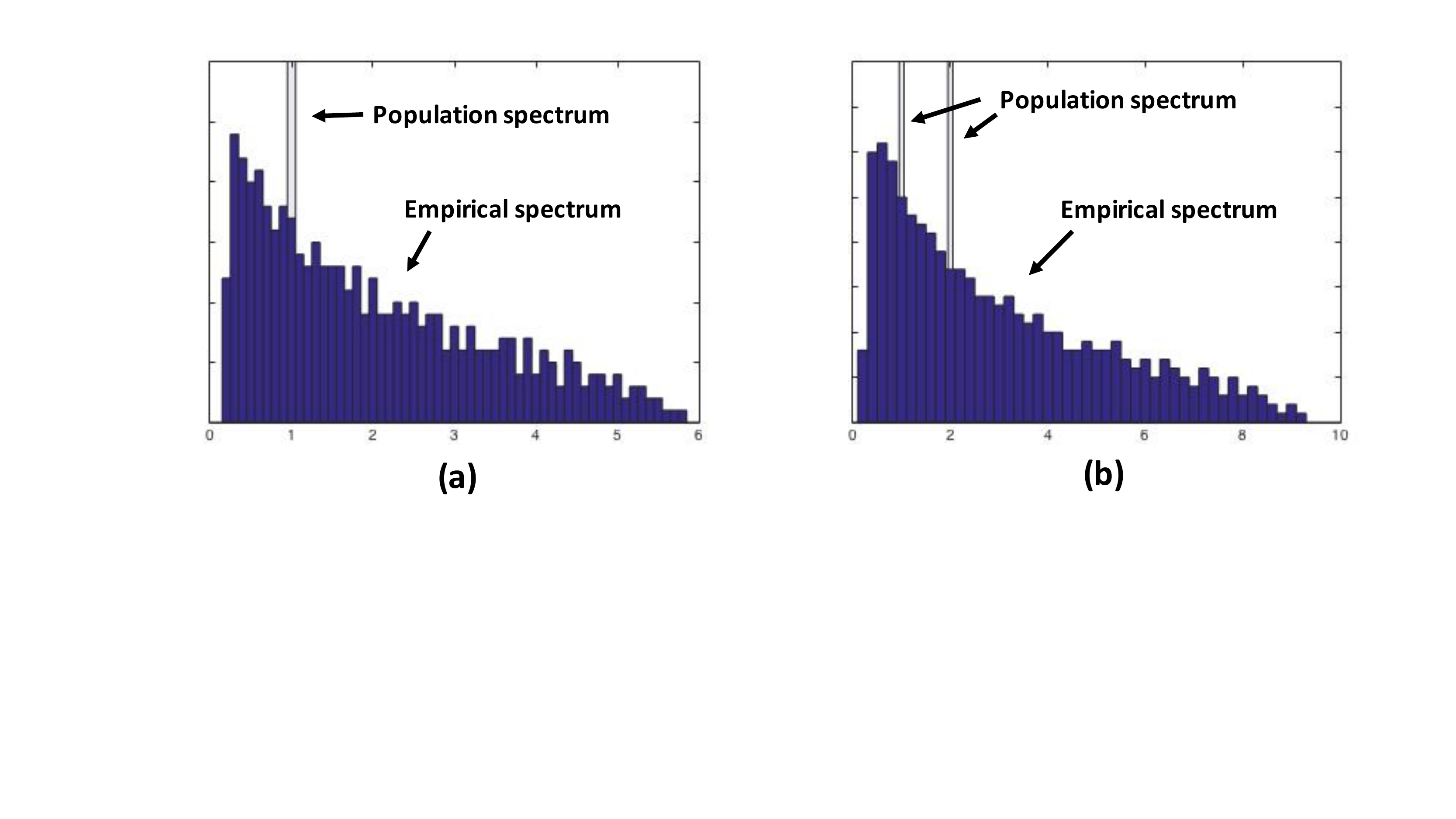}
\vspace{-3.3cm}\caption{\small{Plots of the eigenvalues of the empirical covariance, corresponding to $n=500$ samples from: (a) the $d=1000$ dimensional Gaussian with identity covariance, and (b) the $d=1000$ dimensional ``2-spike" Gaussian distribution whose covariance has $500$ eigenvalues with value 1, and $500$ eigenvalues of value 2.  Note that the empirical spectra are poor approximations of the true population spectra.   To what extent can the eigenvalues of the true distribution (the population spectrum) be accurately recovered from the samples, particularly in the ``sublinear'' data regime where $n \ll d$? \label{fig1}}}.
\end{figure}

This question of understanding the relationship between the empirical spectrum and the population spectrum of the underlying distribution has a long history of study, both from the perspective of characterizing the empirical spectrum, and with the goal of correcting for its biases.  In the former vein, the seminal work of Anderson~\cite{anderson1963asymptotic} considered the joint distribution of the empirical eigenvalues in the asymptotic regime as $n$ tends towards infinity, for fixed dimension, $d$.  The work of Marcenko and Pastur~\cite{marvcenko1967distribution,silverstein1995strong} and more recent advances in random matrix theory have enabled analysis of the empirical spectrum, particularly in the asymptotic regime where the dimension and sample size scale linearly with each other (see e.g. Bai and Silverstein's recent book~\cite{bai2010spectral}).  We provide a more detailed discussion of the relationship between this characterization of the empirical spectrum and the problem of recovering the population spectrum in Section~\ref{sec:relatedwork}.

In the latter vein, works have considered both the end objective of recovering the population spectrum, as well as the objective of estimating the population covariance matrix. 
In their seminal work~\cite{james1961estimation}, James and Stein's proposed a shrinkage estimator for covariance estimation which uses the empirical eigenvectors, but ``shrinks'' the empirical eigenvalues to reduce the overall error due to the differences between the empirical and population spectra.  Takemura~\cite{takemura1984orthogonally},  and Dey and Srinivasan~\cite{dey1985estimation} extended this work of James and Stein, obtaining orthogonally invariant minimax covariance estimators.  There are many other approaches to eigenvalue shrinkage in this early line of work, e.g.~\cite{efron1976multivariate,stein1975estimation,stein1977lectures,haff1980empirical, haff1979identity}. 
More recently, there has been a significant effort to develop optimal covariance estimators in the asymptotic regime where both $n$ and $d$ tend to infinity.  This includes work of Ledoit and Wolf~\cite{ledoit2004well,ledoit2011nonlinear,ledoit2011eigenvectors}, Sch{\"a}fer and Strimmer~\cite{schafer2005shrinkage}, and more recent work of Donoho et al.~\cite{donoho2013optimal} who considered shrinkage estimators in the spiked covariance model (when all population eigenvalues take a constant number of different values). 

While both the problems of covariance estimation and spectrum estimation face the common challenge that the empirical spectrum might differ significantly from the population spectrum, the problems are different.  It is not clear whether optimal estimators for one of the problems can be leveraged to yield optimal or near-optimal estimators for the other problem.   After formally defining the specific problem that we tackle---estimating the population spectrum in the linear and sublinear data regimes---in Section~\ref{sec:relatedwork}, we provide a technical discussion of the more modern related work on spectrum estimation, beginning with the seminal work of Karoui~\cite{karoui2008} and Burda~\cite{burda2004signal,burda2005spectral}.


\subsection{Setup and Definition}
We focus on the general setting where the multivariate distribution over $\mathbb{R}^d$ in question is defined by a real-valued distribution $X$, with zero mean, variance 1, and fourth moment $\fourthm$, and a real $d \times d$ matrix $\S$.   A sample of $n$ vectors, viewed as a $n \times d$ data matrix $\Y$ consisting of $n$ vectors drawn independently from the distribution corresponding to the pair $(X,\S)$ is given by $\Y = \X \S$ where $\X \in \mathbb{R}^{n \times d}$ has i.i.d entries drawn according to distribution $X$.   Note that this setting encompasses the case where the data is drawn from the uniform distribution over a $d$-dimensional unit cube, and the case of a multivariate Gaussian (corresponding to the distribution $X$ being the standard Gaussian and the covariance of the corresponding multivariate Gaussian given by $\S^T \S$).

Throughout, we denote the corresponding \emph{population covariance} matrix $\Sigma = \S^T \S$, and its eigenvalues by $\bm{\lambda} = \lambda_1,\ldots,\lambda_d,$  with $\lambda_1 \le \lambda_2 \le \ldots \le \lambda_d.$   Our objective will be to recover an accurate approximation to this sorted vector of eigenvalues, $\bm{\lambda},$ given a data matrix $\Y$ as defined above.  It is also convenient to regard the vector of eigenvalues as a distribution over $\mathbb{R}$, consisting of $d$ equally-weighted point masses at locations $\lambda_1,\ldots,\lambda_d$; we refer to this distribution as the \emph{population spectral distribution} $D_{\Sigma}$.   We note that the task of learning the sorted vector $\bm{\lambda}$ in $L_1$ distance is closely related to the task of learning the spectral distribution in \emph{Wasserstein} distance (i.e. ``earthmover distance''):  the $L_1$ distance between two sorted vectors of length $d$ is exactly $d$ times the Wasserstein-1 distance between the corresponding point-mass distributions.  Similarly, given a distribution, $Q$, that is close to the true spectral distribution $D_{\Sigma}$ in Wasserstein distance, the length $d$ vector whose $i$th element is given by the $i$th $(d+1)$--quantile\footnote{For $i =1,\ldots,d,$ the $i$th $(d+1)$--quantile of a distribution $Q$ is defined to be the minimum value, $x$, with the property that the cumulative distribution function of $Q$ at $x$ is at least $i/(d+1)$.} of $Q$ will be close, in $L_1$ distance, to the sorted vector of population eigenvalues.

\subsection{Related Work}\label{sec:relatedwork}

Before formally stating our main result of accurate spectrum estimation in the sublinear data regime, we discuss the context of our results and the connections to existing related work on spectrum estimation.\\

\noindent \textbf{Population and Sample Spectra: the Marcenko-Pastur Law.}
Given the setting described above, where we observe an $n \times d$ data matrix $\Y= \X \S$ with population covariance given by $\S^T \S$, in the regime in which the dimensionality of the data, $d$, is linear in the number of samples, $n$, much is known about the mapping from the population spectral distribution $D_{\Sigma}$, to the empirical spectral distribution of the samples.  Specifically, provided the ratio of the number of samples, $n$, to the dimensionality of the samples, $d$, is bounded below by some constant $\gamma > 0$, for sufficiently large $n,d$, the \emph{expected} empirical spectral distribution will be well approximated by a deterministic function of the population spectral distribution.  This deterministic function characterizing the correspondence between the empirical and population spectra is known as the Marcenko-Pastur Law, which is defined in terms of the Stieltjes transform (also referred to as the \emph{Cauchy transform}) of the spectral distribution~\cite{marvcenko1967distribution, silverstein1995strong}.  At least in the linear regime in which $n$ and $d$ scale together, it is not hard to show that the empirical spectral distribution will be close to the \emph{expected} empirical spectral distribution, and hence, asymptotically, the Marcenko-Pastur law will give an accurate characterization of the empirical spectral distribution.  We refer the reader to chapter 3 of Bai and Silverstein's book~\cite{bai2010spectral} for a thorough treatment of the Stieltjes transform and Marcenko-Pastur law. \\ 

\noindent \textbf{Inverting the Marcenko-Pastur Law.}
Perhaps the most natural approach to recovering the population spectrum from the data matrix $\Y$ is to attempt to invert the mapping between population spectrum and expected empirical spectrum given by the Marcenko-Pastur law.  The seminal work of Karoui~\cite{karoui2008} shows that this inversion can be represented via a linear program, and that, in the linear regime where $n/d \rightarrow \gamma \in (0,\infty)$, the reconstruction will be asymptotically consistent.  This work also demonstrates the practical viability of this approach on a series of synthetic data, for the setting $d=100,n=500.$   Building off the work of Karoui, Li et al.~\cite{li13} considered applying this approach to a parametric  model where the population spectral distribution has a constant (finite) support.   They also suggested extending the Marcenko-Pastur law to the real line, allowing the optimization to be conducted over the reals, which makes the  optimization procedure both easier to implement and more computationally efficient. Another approach to invert the Marcenko-Pastur law directly, proposed by Ledoit and Wolf~\cite{ledoit2011nonlinear,ledoit2013spectrum}, exploits the natural discreteness of the population spectrum in finite dimensions, and optimized over the Marcenko-Pastur law on the real line.  Simulations demonstrated that this approach yields significant improvements in the accuracy of the recovered spectrum, versus the earlier approach of Karoui~\cite{karoui2008}.

In a similar spirit, Mestre~\cite{Mestre08} considered the task of recovering the population spectrum in the setting where the population spectral distribution has a constant support (say of size $r$) and where the weights (but not the values) of each point mass are known a priori.   Mestre proposed an algorithm for recovering the values of the support of the population spectrum via inverting the Marcenko Pastur law, which is successful provided the empirical spectral distribution consists of $r$ clusters of values, corresponding to the $r$ point masses of the population spectrum.   Provided sufficient separation between the point masses of the population spectrum, in the linear regime where $n$ and $d$ have constant ratio, the requirements of the algorithm are provably satisfied.

There seem to be two limitations to this general approach of ``inverting'' the Marcenko-Pastur law.  The first is that the Marcenko-Pastur law, in general, is poorly equipped to deal with the sublinear-sample regime where $n \ll d$.  In this sublinear sample-size regime, for example, with $n = d^{2/3}$, even if the population spectrum has a specific limiting distribution, the expected empirical spectral distribution may not converge.  The second drawback is the difficulty of obtaining theoretical bounds on the accuracy of the recovered spectral distribution.  This seems mainly due to the difficulty of analyzing the robustness of inverting the Marcenko-Pastur law.  Specifically, given a dimension and sample size, it seems difficult to characterize the set of population spectral distributions that map, via the Marcenko-Pastur law, to a given neighborhood of a specific empirical spectral distribution.  This analysis is further complicated in the sublinear sample regime by the potential lack of concentration of the empirical spectrum.  \\

\noindent \textbf{Method of Moments.}
There have been several works that approach the spectrum recovery problem via the method of moments~\cite{Rao08,Bai10,Li14}.  Rao et al.~\cite{Rao08} observed the fact that the moments of the empirical spectral distribution have a limiting Gaussian distribution whose mean and variance are functions of the population spectrum. Given these moments distributions, they proposed a maximum likelihood approach to recover the parameters of the population spectrum in the setting where the spectrum consists of a constant number of point masses.   In a similar fashion, Bai et al~\cite{Bai10} directly estimate the moments of the population spectrum from the empirical moments, via a system of polynomial equations that is derived from the Marcenko-Pastur Law.  In the linear sample-size setting, Bai et al. show that their recovery is consistent.   We note that an immediate consequence of our accurate moment estimation (Theorem~\ref{thm:estmom}), together with the fact that a distribution supported on at most $r$ values is robustly determined by its first $2r-1$ moments (see e.g.~\cite{Bai10}), yields the fact that for such spectral distributions, consistent estimation is possible in the sublinear data regime as long as $\frac{n}{d^{1-\frac{2}{2r-1}}} \rightarrow \infty.$

The recent work of Li and Yao~\cite{Li14} essentially interpolates between the approach of Mestre~\cite{Mestre08} and Bai et al.~\cite{Bai10} to tackle the setting where the spectrum consists of a constant, $r$, number of point masses, but where the empirical spectrum can not be partitioned into $r$ corresponding clusters.  For these ``mixed'' clusters, they employ the moment-based approach of Bai et al., and show consistency in the linear sample-size regime.

Finally, the work of  Burda et al.~\cite{burda2004signal} from the physics community employs a method of moment approach to recovering specific classes of population spectra---for example the $3$-spike case.  This work is essentially a method-of-moments approach to inverting the Marcenko-Pastur law in specific cases, although this work seems to be unaware of the Marcenko-Pastur law and the related literature relating the empirical and population spectra.\\

\noindent \textbf{Sketching Bi-Linear Forms.}
In a recent work~\cite{li2014sketching}, Li et al. consider a seemingly unrelated problem, the problem of \emph{sketching} matrix norms.  Namely, suppose one wishes to approximate the $k$th moment of the spectrum of a $d\times d$ matrix, $\Sigma$, $||\Sigma||^k_k = \sum_{i=1}^d \lambda_i^k$,  but rather than working directly with the matrix $\Sigma,$ one only has access to a much smaller matrix that is a bilinear \emph{sketch} of $\Sigma$.  The question is how to design this sketch: for some $r,s \ll d$, can one design distributions $\mathcal{A}$ and $\mathcal{B}$ over $r \times d$ and $d \times s$ matrices, respectively, such that for any $\Sigma$, with high probability, given matrices $\A$ and $\B$ drawn respectively from $\mathcal{A}$ and $\mathcal{B}$, $||\Sigma||_k$ can be approximated based on the $r\times s$ matrix $\A \Sigma \B$?   The authors consider setting $\mathcal{A}$ and $\mathcal{B}$ to have i.i.d Gaussian entries, and show that such sketches are information theoretically optimal, to constant factors. 

The connection between sketching matrix norms and recovering moments of the population covariance is that the matrix $\Y \Y^T = \X \S \S^T \X^T$ can be viewed as a bilinear sketch of the matrix $\S \S^T$.   While $\S \S^T$ is not the population covariance matrix, it has the same eigenvalues (and hence same moments) as the population covariance.

The main difference between our work and the work of Li et al. is the conceptual difference in focus: we are focussed on recovering the population spectrum from limited data, they are focussed on defining small sketching matrices for matrix norms.  The approach to moment recovery of~\cite{li2014sketching} and our work both leverage a simple unbiased moment estimator (see Fact~\ref{fact:unbiased}), though our techniques differ in two ways: first, ~\cite{li2014sketching} is concerned with establishing the minimum sketch size from an information theoretic perspective, and the proposed algorithm is not computationally efficient; second, from a technical perspective, the proof of correctness in~\cite{li2014sketching} focusses on the Gaussian setting, and it seems difficult to extend their analysis techniques to the more general setting that we consider.  In particular, to prove the variance bound of the moment estimator in the more general setting,  we take a rather different route and employ a variant of the approach of Yin and Krishnaiah~\cite{yin1983limit}. \\ 

\noindent \textbf{Other Works on Spectrum Reconstruction.}
There are also several other works on the population spectrum recovery problem for specific classes of population covariance.  These include the paper of Bickel and Levina~\cite{bickel2008regularized} who obtain accurate reconstruction in the sublinear-sample setting for the class of population covariance matrices whose off-diagonal entries decrease quickly with their distance to the diagonal (for example, as in the class of Toeplitz matrices).  

\subsection{Summary of Approach and Results}

Our approach to recovering the population spectral distribution from a given data matrix is via the method of moments, and is motivated by the observation (also leveraged in~\cite{li2014sketching}) that there is a natural unbiased estimator for the $k$th moment of the population spectral distribution:

\begin{fact}\label{fact:unbiased}
Fix a list of $k$ distinct integers $\sigma = (\sigma_1,\ldots,\sigma_k)$ with $\sigma_i \in \{1,\ldots,n\}$.   Let $\Y = \X \S$ where $\X \in \mathbb{R}^{n \times d}$ consists of entries drawn i.i.d from a distribution of zero mean and variance 1, and $\S \in \R^{d \times d}$.  Letting $\A = \Y \Y^T$, we have $$\E \left[ \prod_{i=1}^k \A_{\sigma_i,\sigma_{(i \text{ mod } k) + 1}} \right] = \sum_{i=1}^d \lambda_i^k,$$ where the expectation is over the randomness of the entries of $\X$, and $\lambda_i$ is the $i$th eigenvalue of the population covariance matrix $\S^T \S.$
\end{fact}

The above fact, whose simple proof is given in Section~\ref{sec:estmom}, suggests that a good algorithm for estimating the $k$th moment of the spectral distribution would be to compute the sum of the above quantity over \emph{all} sets $\sigma$ of distinct indices.  The naive algorithm for computing such a sum would take time $O(n^k)$ to evaluate, and it seems unlikely that a significantly faster algorithm exists.\footnote{The ability to efficiently compute this sum would imply an efficient algorithm for counting the number of $k$-cycles in a graph, which is NP-hard, for general $k$~\cite{alon1997finding}.}  Fortunately, as we show, there is a simple algorithm that computes the sum over all \emph{increasing} lists of $k$ indices;  additionally, such a sum results in an estimator with comparable variance to the computationally intractable estimator corresponding to the sum over all lists.  This algorithm, together with a careful analysis of the variance of the corresponding estimator, yields the following theorem:

\begin{theorem}[Efficient Moment Estimation]\label{thm:estmom}
There is an algorithm that takes $\Y=\X \S$ and an integer $k \ge 1$ as input, runs in time $poly(n,d,k)$ and with probability at least $1-\delta$, outputs an estimate of $\frac{1}{d}\|\S^T\S\|_k^k$ (i.e. $\frac{1}{d}\sum_i \lambda_i^k$) with multiplicative error at most
$$
\frac{f(k)}{\sqrt{\delta}} \max( \frac{d^{k/2-1}}{n^{k/2}}, \frac{d^{\frac{1}{4}-\frac{1}{2k}}}{\sqrt{n}},\frac{1}{\sqrt{n}}),
$$
where the function $f(k)=2^{6k}k^{3k}\fourthm^{k/2}$.
\end{theorem}

Restated slightly, the above theorem shows that the $k$th moment of the population spectrum can be accurately computed in the sublinear data regime, provided $n \ge c_k d^{1-\frac{2}{k}},$ for some constant $c_k$ dependent on $k$.  In the asymptotic regime as $d \rightarrow \infty$, this theorem implies that the multiplicative error of the estimate of the $k$th moment goes to zero provided $\frac{n}{d^{1-\frac{2}{k}}} \rightarrow \infty$.   This moment recovery is useful in its own right, as these moments of the spectral distribution (also referred to as the \emph{Schatten} matrix norms of the population covariance) provide insights into the population distribution (see, e.g.~\cite{hardt2012simple} and the survey~\cite{mahoney2011randomized}).

The recovery guarantees of Theorem~\ref{thm:estmom} are optimal to constant factors: to accurately estimate the $k$th moment of the population spectrum to within a constant multiplicative error, the sample size $n$ must scale at least as $d^{1-2/k}$, as is formalized by the following lower bound, which is a corollary to the lower bound in~\cite{li2014sketching}.  

\begin{cor}\label{cor:lb}
Fix a constant integer $k$, and suppose there exists an algorithm that, for any $d \times d$ matrix $\S$, when given an $n \times d$ data matrix $\Y = \X\S$ with entries of $\X$ chosen i.i.d. as above, outputs an estimate $y$ satisfying the following with probability at least $3/4$: $0.9\|\S\S^T\|^k_k\leq y\leq 1.1\|\S\S^T\|^k_k$; then $n \ge c d^{1-2/k},$ for an absolute constant $c$ independent of $n,d,$ and $k$.
\end{cor}

Given accurate estimates of the low-order moments of the population spectral distribution, an accurate approximation of the list of population eigenvalues can be recovered by first solving the moment inverse problem---namely finding a distribution $D$ whose moments are close to the recovered moments, for example, via linear programming---and then returning the vector of length $d$ whose $i$th element is given by the $i$th $(d+1)$-quantile of distribution $D$.  Altogether, this yields a practically viable polynomial-time algorithm with the following theoretical guarantees for recovering the population spectrum:

\begin{theorem}[Main Theorem]\label{thm:main}
Consider an $n \times d$ data matrix $\Y = \X\S$, where $\X\in R^{n\times d}$ has i.i.d entries with mean $0$, variance $1$, and fourth moment $\fourthm$, and $\S$ is a real $d \times d$ matrix s.t. the eigenvalues of the population covariance $\Sigma= \S^T \S$, $\bm{\lambda} = \lambda_1,\ldots, \lambda_d,$ are upper bounded by a constant $b$. There is an algorithm that takes $\Y$ as input and for any integer $k\geq 1$ runs in time $poly(n,d,k)$ and outputs $\bm{\hat{\lambda}} = \hat{\lambda}_1,\ldots,\hat{\lambda}_d$ with expected $L_1$ error satisfying: 
$$
\E\left [\sum_{i=1}^d |\lambda_i-\hat{\lambda_i}|\right ] \le b d\left (f(k) (\frac{d^{k/2-1}}{n^{k/2}}+\frac{1+d^{\frac{1}{4}-\frac{1}{2k}}}{n^{1/2}}) +\frac{C}{k}+\frac{1}{d} \right ),
$$
 where $C$ is an absolute constant and $f(k)=C' (6 k)^{3k+1}\fourthm^{k/2}$ for an absolute constant $C'.$   
\end{theorem}

This theorem implies that our population spectrum estimator is asymptotically consistent in terms of Wasserstein distance, even in the \emph{sublinear} sample-size regime where $\frac{d}{n} \rightarrow \infty$:

\begin{cor}[Consistent Sublinear Sample-Size Estimation]\label{cor:consistency}
Fix a limiting spectral distribution $p_\infty$ that is absolutely bounded by a constant, and a sequence of absolutely bounded population spectral distributions, $p_1,p_2,\ldots$ and corresponding population covariance matrices $\Sigma_1,\Sigma_2,\ldots,$ such that $p_d$ is the spectral distribution of $\Sigma_d$, and $p_d$ converges weakly to $p_\infty$ as $d\rightarrow \infty$.  Given a sequence of data matrices, with the $d$th matrix $\Y_d = \X_d\S_d$ being $n_d \times d$ with $\S_d^T\S_d = \Sigma_d$ and entries of $\X_d$ chosen i.i.d with zero mean, variance 1, and bounded fourth moment, then our algorithm outputs a distribution $q_d$ on input $\Y_d$ such that  $q_d$ converges weakly to $p_{\infty}$, provided $\frac{n_d}{d^{1-\eps}} \rightarrow \infty$ for every positive constant $\epsilon.$  (For example, taking $n_d = \frac{d}{\log d}$ yields asymptotically consistent sublinear sample spectrum estimation.)
\end{cor}

The proof of Theorem~\ref{thm:main} follows from combining Theorem~\ref{thm:estmom} with the following proposition that bounds the Wasserstein distance between two distributions in terms of their discrepancies in low-order moments:
\begin{proposition}\label{thm:redb}
Given two distributions with respective density functions $p, q$ supported on $[-1,1]$ whose first $k$ moments are $\bm{\alpha} = (\alpha_1,...\alpha_k)$ and $\bm{\beta} = (\beta_1,...\beta_k)$, respectively, the Wasserstein distance, $W_1(p,q)$, between $p$ and $q$ is bounded by: 
$$
W_1(p,q) \le \frac{C}{k}+g(k)\|\bm{\alpha}-\bm{\beta}\|_2, 
$$ where $C$ is an absolute constant, and $g(k)=C' 3^k$ for an absolute constant $C'$.
\end{proposition}

The proof of the above proposition proceeds by leveraging the dual definition of Wasserstein distance: $W_1(p,q) = sup_{f \in Lip} \int_{-\infty}^\infty f(x)\left(p(x)-q(x)\right) dx,$ where $Lip$ denotes the set of all Lipschitz-1 functions.   Our proof argues that for any Lipschitz function $f$, after convolving it with a special ``bump'' function, $\hat{b}$, which is a scaled Fourier transform of the bump function used in $\cite{kane2010exact}$, the resulting function $f * \hat{b}$ has small high-order derivatives and is close to $f$ in $L_{\infty}$ norm.  Given the small high-order derivates of $f * \hat{b},$ there exists a good degree-$k$ polynomial interpolation of this function, $P_k$: the closeness of the first $k$ moments of $p$ and $q$ implies a bound on the integral $\int P_k(x) \left(p(x)-q(x)\right) dx$, from which we derive a bound on the original Wasserstein distance.  Our approach to approximating a Lipschitz-1 function with a degree $k$ polynomial can also be seen as a constructive proof of a special case of Jackson's Theorem (see e.g. Theorem 7.4 in~\cite{carothers1998short}).

We also show, via a Chebyshev polynomial construction, that the inverse linear dependence of Proposition~\ref{thm:redb} between the number of moments $k$, and the Wasserstein distance between the distributions, is tight in the case that the moments exactly match:
\begin{proposition}\label{thm:emLowerBound}
For any even $k$, there exits a pair of distributions $p$, $q$, each consisting of $k/2$ point masses, supported within the unit interval $[-1,1]$, s.t. $p$ and $q$ have identical first $k-2$ moments, and Wasserstein distance $W_1(p,q) > 1/2k$.
\end{proposition}

\subsection{Organization of Paper}
In Section~\ref{sec:estmom} we motivate and state our algorithms for accurately recovering the moments of the population spectrum, and prove Theorem~\ref{thm:estmom}.  The most cumbersome component of this proof of correctness of our algorithm is the proof of a bound on the variance of our moment estimator; we defer this proof to Appendix~\ref{ap:varb}.  In Section~\ref{sec:thmmain}, we state our algorithm for leveraging accurate moment reconstruction to recover the population spectrum, and describe the connection between the Wasserstein distance between spectral distributions, and $L_1$ distance between the vectors.  In Section~\ref{sec:emd}, we establish Propositions~\ref{thm:redb} and~\ref{thm:emLowerBound}, which bound the Wasserstein distance between two distributions in terms of their discrepancies in low-order moments, completing our proof of Theorem~\ref{thm:main}.  Section~\ref{sec:simulation} contains some results illustrating the empirical performance of our approach.

\section{Estimating the Spectral Moments}\label{sec:estmom}

The core of our approach to recovering the moments of the population spectral distribution is a convenient \emph{unbiased} estimator for these moments, first proposed in the recent word of Li et al.~\cite{li2014sketching} on sketching matrix norms.  This estimator is defined via the notion of a \emph{cycle}:

\begin{definition}
Given integers $n$ and $k$, a $k$-\emph{cycle} is a sequence of $k$ distinct integers, $\sigma = (\sigma_1,...\sigma_k)$ with $\sigma_i \in [n].$   Given an $n\times n$ matrix $A$, each cycle, $\sigma$, defines a product: 
$$
\A_{\sigma} = \prod_{i=1}^k\A_{\sigma_i,\sigma_{i+1}},
$$ with the convention that $\sigma_{k+1} = \sigma_1$, for ease of notation.
\end{definition}

The following observation demonstrates the utility of the above definition:

\begin{fact}\label{thm:incyc}
For any $k$-cycle $\sigma$, a symmetric $d \times d$ real matrix $\T$, and a random $n \times d$ matrix $X$ with i.i.d entries with mean 0 and variance 1, $$\E[(\X^T\T\X)_{\sigma}] = trace(\T^k),$$ where the expectation is over the randomness of $X$.
\end{fact}
\begin{proof}
We can expand $\E[(\X^T\T\X)_{\sigma}]$ as following, where $\gamma_{k+1}$ is shorthand for $\gamma_1$,  $$\sum_{\delta_1,...,\delta_k, \gamma_1,\ldots,\gamma_k \in [d]}\E[\prod_{i=1}^k \X_{\delta_{i}, \sigma_i} \T_{\delta_i, \gamma_{i+1}} \X_{\gamma_{i+1}, \sigma_{i+1}}] 
= \sum_{\delta_1,...\delta_{k}} \prod_{i=1}^k \T_{\delta_{i},\delta_{i+1}} = trace(\T^k).
$$
The first equality holds since, for every term of the expression, the expectation of that term is zero unless each of the entries of $\X$ appears at least twice.   Because the $\sigma_i$ are distinct, in every nonzero term, each of the entries of $\X$ will appear exactly twice and $\delta_i = \gamma_i$.
\end{proof}

The above fact shows that each $k$-cycle yields an unbiased estimator for the $k$th spectral moment of $T$.  While each estimator is unbiased, the variance will be extremely large.  Perhaps the most natural approach to reducing this variance, would be to compute the average over \emph{all} $k$-cycles.  Unfortunately, such an estimator seems intractable, from a computational standpoint.   The naive algorithm for computing this average---simply iterating over the ${n \choose k}$ different $k$-cycles---would take time $O(n^k)$ to evaluate.  It seems unlikely that a significantly faster algorithm exists, assuming that $P \neq NP$, as an efficient algorithm to compute this average over $k$-cycles would imply an efficient algorithm for counting the number of simple $k$-cycles in a graph (i.e. loops of length $k$ with no repetition of vertices), which is known to be NP-hard for general $k$ (see, e.g.~\cite{alon1997finding}). 

One computationally tractable variant of this average over all $k$-cycles, would be to relax the condition that the $k$ elements of each cycle be distinct.  This quantity is simply the trace of the matrix $\left(\X^T\T\X\right)^k$, which is trivial to compute!  Unfortunately, this \emph{exactly} corresponds to the $k$th moment of the empirical spectrum, which is a significantly biased approximation of the population spectral moment (for example, as illustrated in Figure~\ref{fig1}).

Our algorithm proceeds by computing the average of all \emph{increasing} $k$-cycles:
\begin{definition}
An \emph{increasing} $k$-cycle $\sigma = (\sigma_1,...\sigma_k)$ is a $k$-cycle with the additional property that $\sigma_1 < \sigma_2 < \ldots < \sigma_k.$
\end{definition}

We observe that, perhaps surprisingly, there is a simple and computationally tractable algorithm for computing the average over all increasing cycles.  Given $\Y = \X \S$, instead of computing the trace of $(\Y^T \Y)^k$, which would correspond to the empirical $k$th moment, we instead zero out the diagonal and lower-triangular entries of $\Y^T \Y$ in the `first' $k-1$ copies of $\Y^T \Y$ in the product $(\Y^T \Y)^k$.  It is not hard to see that this exactly corresponds to preserving the set of increasing cycles, as the contribution to a diagonal entry of the product corresponding to a non-increasing cycle will include a lower-triangular entry of one of the terms, and hence will be zero (see Lemma~\ref{lem:inccyc} for a formal proof).   This motivates the following algorithm for estimating the $k$th moment of the population spectrum:

\begin{algorithm}[H]
\begin{algorithmic}
\STATE \textbf{Input:} $Y\in R^{n\times d}$\\
\STATE Set, $A = Y Y^T,$ and let $G = A_{up}$ be the matrix $A$ with the diagonal and lower triangular entries set to zero.
\STATE \textbf{Output:} $\frac{tr(G^{k-1}A)}{d\cdot \binom{n}{k}}$\\
\end{algorithmic}
\caption{[Estimating the $k$th Moment]\label{alg:moment}}
\end{algorithm}

Our main moment estimation theorem characterizes the performance of the above algorithm:

\vspace{.5cm}\noindent \textbf{Theorem~\ref{thm:estmom}.} \emph{
Given a data matrix $\Y=\X \S$ where the entries of $X$ are chosen i.i.d. with mean 0, variance 1, and fourth moment $\fourthm$, Algorithm~\ref{alg:moment} runs in time $poly(n,d,k)$ and with probability at least $1-\delta$, outputs an with estimate of $\frac{1}{d}\|\S^T\S\|_k^k$ (i.e. $\frac{1}{d}\sum_i \lambda_i^k$) with multiplicative error at most
$$
\frac{f(k)}{\sqrt{\delta}} \max( \frac{d^{k/2-1}}{n^{k/2}}, \frac{d^{\frac{1}{4}-\frac{1}{2k}}}{\sqrt{n}},\frac{1}{\sqrt{n}}),
$$
where the function $f(k)=2^{6k}k^{3k}\fourthm^{k/2}$.} \medskip

The following restatement of the above theorem emphasizes the fact that accurate estimation of the population spectral moments is possible in the sublinear sample regime where $n = o(d)$.

\begin{cor}\label{cor:conerr}
Suppose $X$ is a random $n\times d$ matrix whose entries are chosen i.i.d as described above. For any constant $c>1$, there exists a function $f_c(k)$ such that, given $n=f_c(k)d^{1-2/k}$, for any $d\times d$ real matrix $S$, Algorithm~\ref{alg:moment} takes data matrix $\Y=\X \S$ as input, runs in time $poly(d,k)$ and with probability at least $3/4$ (over the randomness of $\X$), outputs an estimate, $y$, of the $k$th population spectral moment that is a multiplicative approximation in the following sense: $$\left(1 - \frac{c^{2k}-1}{c^{2k}+1}\right)\|\S^T\S\|_k^k \le y \le \left(1 + \frac{c^{2k}-1}{c^{2k}+1}\right)\|\S^T\S\|_k^k.$$
\end{cor}

As the above corollary shows, for any constant integer $k \ge 1$ there is a constant $c_{k}$ such that taking $n = c_{k} d^{1-2/k}$ is sufficient to estimate the $k$th spectral moment accurately.  For any constant $k$, this sublinear dependence between $n$ and $d$ is information theoretically optimal in the following sense (which is a stronger version of Corollary~\ref{cor:lb}):

\begin{proposition}\label{prop:conerr}
Given the setting of Theorem~\ref{thm:estmom}, for any $k>1$, suppose that an algorithm takes $\X\S$ and with probability at least $3/4$ computes $y$ with $(1-\epsilon)\|\S\S^T\|^k_k\leq y\leq (1+\epsilon)\|\S\S^T\|^k_k$ for $\epsilon=(1.2^{2k}-1)/(1.2^{2k}+1)$. Then $\X$ must be $n \times d$ for $n \ge c d^{1-2/k}$ for an absolute constant $c$.
\end{proposition}

 The above proposition follows as an immediate corollary from Theorem 3.2 of \cite{li2014sketching} by plugging in $S=X$, $T=I_d$, $A = S$ and $p=2k$.

\subsection{Proof of Theorem~\ref{thm:estmom}}\label{sec:estmomproof}
The proof of this theorem follows from the following three components:  Lemma~\ref{lem:inccyc} (below) shows that the efficient Algorithm~\ref{alg:moment} does in fact compute the average over all increasing $k$-cycles, $(\Y \Y^T)_\sigma$; Fact~\ref{thm:incyc} guarantees that the average over such cycles is an unbiased estimator for the claimed quantity; and Proposition~\ref{variancebound} bounds the variance of this estimator, which, by Chebyshev's inequality, guarantees the claimed accuracy of Theorem~\ref{thm:estmom}.  Our proof of this variance bound follows a similar approach as in~\cite{yin1983limit}.

The following lemma shows that Algorithm~\ref{alg:moment} computes the average over all increasing $k$-cycles, $\sigma$, of $(\Y \Y^T)_{\sigma};$ for an informal argument, see the discussion before the statement of Algorithm~\ref{alg:moment}.

\begin{lemma}\label{lem:inccyc}
Given an $n\times d$ data matrix $\Y = \X \S$, Algorithm~\ref{alg:moment} returns the average of $(\Y \Y^T)_{\sigma}$ taken over all increasing $k$-cycles, $\sigma$.
\end{lemma}
\begin{proof}
Let $A=\Y \Y^T = \X \S \S^T \X^T$; let $U_{i,j,m}$ denote the set of increasing $m$-cycles $\sigma$ such that $\sigma_1 = i$ and $\sigma_m = j$, and define $$F_{i,j,m} = \sum_{\sigma \in U_{i,j,m}} \prod_{\ell = 1}^{m-1}A_{\sigma_\ell,\sigma_{\ell+1}}.$$
There is a simple recursive formula of $F_{i,j,m}$, given by
\begin{align}
F_{i,j,m} = \sum_{\ell=i}^{j-1}F_{i,\ell,m-1}A_{\ell,j} \label{dp:formula}
\end{align}
 Let $G$ be the strictly upper triangular matrix of $A$, as in Algorithm~\ref{alg:moment}, and let $F^{(m)}$ denote the matrix whose $(i,j)$th entry is $F_{i,j,m}$.  The recursive formula~\ref{dp:formula} can be rewritten as $F^{(m)}=F^{(m-1)}G$.  Given this, the sum over all increasing $k$-cycles is $$\sum_{i,j} F_{i,j,k-1}A_{j,i} = tr(F^{(k)}A)=tr(G^{k-1}A),$$ as claimed.
\end{proof}

The main technical challenge in establishing the performance guarantees of our moment recovery, is bounding the variance of our (unbiased) estimator.  

\begin{proposition}\label{variancebound}
Given the setup of Theorem~\ref{thm:estmom}, let $U$ be the set of all increasing cycles of length $k$, then the following variance bound holds, where the function $f(k)=2^{12k}k^{6k}\fourthm^k$ :
$$
 \Var[\frac{1}{|U|}\sum_{\sigma\in U}(\X^T \T\X)_{\sigma}]\le f(k) \max( \frac{d^{k-2}}{n^k}, \frac{d^{\frac{1}{2}-\frac{1}{k}}}{n},\frac{1}{n})tr(\T^k)^2.
 $$
\end{proposition}

To see the high level approach to our proof of this proposition,  consider the following:  given lists of indices $\delta = (\delta_1,...\delta_{k})$ and $\gamma=(\gamma_1,\ldots,\gamma_k),$ with $\delta_i,\gamma_i \in [d]$, we have the following equality:
$$
(\X^T\T\X)_\sigma = \sum_{\delta,\gamma \in [d]^k}\prod_{i=1}^k \X_{\sigma_i\delta_{i}}\T_{\delta_{i},\gamma_{i+1}} \X_{\sigma_{i+1},\gamma_{i+1}}.
$$

We now seek to bound each cross-term in the expansion of the total variance $\Var[\sum_{\sigma \in U} (\X^T \T \X)_\sigma]$:  namely, for a pair of increasing $k$-cycles, $\sigma, \pi$ consider their contribution to the variance $\E[(\X^T \T\X)_{\sigma}(\X^T \T\X)_{\pi}] $ being
$$\sum_{\delta,\delta',\gamma,\gamma' \in [d]^k}\prod_{i=1}^k \T_{\delta_i, \gamma_{i+1}} \T_{\delta'_{i},\gamma'_{i+1}}\prod_{i=1}^k \X_{\sigma_i, \delta_{i}} \X_{\sigma_i, \gamma_{i}} \X_{\pi_i ,\delta'_{i}} \X_{\pi_i, \gamma'_i} 
$$

We bound this sum by partitioning the set of summands, $\{(\delta,\delta',\gamma,\gamma')\}$ into classes.  To motivate the role of these classes, consider the task of computing the expectation of the ``$\X$'' part of the expression, namely $$\E[\prod_{i=1}^k \X_{\sigma_i, \delta_{i}} \X_{\sigma_i, \gamma_{i}} \X_{\pi_i ,\delta'_{i}} \X_{\pi_i, \gamma'_i}],$$ for a given $\delta,\delta',\gamma,\gamma' \in [d]^k$.  Thanks to the i.i.d and zero mean properties of each entry $X_{i,j}$, most of terms are zero.  The idea is to partition the set of summands that give rise to nonzero terms, via the creation of a list of constraints , $L=\{L_1,\ldots,L_m\}$, where each $L_i$ contains only equalities and inequalities(``$\neq$'') involving the indices of $\delta,\delta',\gamma,\gamma'$.  For example, in the case $k=2$, one such constraint could be $L_1=\{\delta_1=\delta_1', \gamma_1=\gamma_1', \delta_2=\gamma_2', \gamma_2=\delta_2'\},$ which specifies a subset of $\{(\delta,\delta',\gamma,\gamma')\}$ that satisfy each of the four specified equalities.  We will design a set of these constraints, $L=\{L_1,\ldots,L_m\}$ satisfying the following useful properties:
\begin{enumerate}
\item Any lists of indices $\delta,\delta',\gamma,\gamma' \in [d]^k$ with the property that the expectation of the $\X$ ``part'' is zero, namely $\E[\prod_{i=1}^k \X_{\sigma_i, \delta_{i}} \X_{\sigma_i, \gamma_{i}} \X_{\pi_i ,\delta'_{i}} \X_{\pi_i, \gamma'_i}] = 0$, does not satisfy any constraint $L_i \in L$.
\item Any lists of indices $\delta,\delta',\gamma,\gamma' \in [d]^k$ whose expectation of the $\X$ ``part'' is non-zero must satisfy exactly one of the constraint.
\item For any constraint $L_i\in L$, all lists of indices $\delta,\delta',\gamma,\gamma' \in [d]^k$ satisfying $L_i$ have the same expected value of the $\X$ ``part'', namely $\E[\prod_{i=1}^k \X_{\sigma_i, \delta_{i}} \X_{\sigma_i, \gamma_{i}} \X_{\pi_i ,\delta'_{i}} \X_{\pi_i, \gamma'_i}]$.
\end{enumerate}

Given a set of constraints, $L$, satisfying the above, the set of summands $\{(\delta,\delta',\gamma,\gamma')\}$ corresponding to a constraint $L_i \in L$ have the same value of $\E[\prod_{i=1}^k \X_{\sigma_i, \delta_{i}} \X_{\sigma_i, \gamma_{i}} \X_{\pi_i ,\delta'_{i}} \X_{\pi_i, \gamma'_i}],$ which will not be too difficult to bound.  What remains is to deal with the $\T$ component of the expression.  Our set of constraints is also useful for this purpose.  For example, consider the following sum over all $(\delta,\delta',\gamma,\gamma')$ that satisfy a  constraint $L_i$:
$$
\sum_{\delta,\delta',\gamma,\gamma'  \text{ s.t. } L_i}\prod_{i=1}^k \T_{\delta_i, \gamma_{i+1}} \T_{\delta'_{i},\gamma'_{i+1}},
$$
the equalities in $L_i$ can be leveraged to simplify the calculation; revisiting the example above with $k=2$, for instance, for the constraint $L_1=\{\delta_1=\delta_1', \gamma_1=\gamma_1', \delta_2=\gamma_2', \gamma_2=\delta_2'\}$, the above expression simply becomes $tr(T^4)$.  

The full details of this partitioning scheme are rather involved, and are given in Appendix~\ref{ap:varb}.


\section{From Moments to Spectrum}\label{sec:thmmain}

Given the accurate recovery of the moments of the population spectral distribution, as described in the previous section, we now describe the algorithm for recovering the population spectrum from these moments.   We proceed via the natural approach to this moment inverse problem. The proposed algorithm has two parts:  first, we recover a distribution whose moments closely match the estimated moments of the population spectrum (recovered via Algorithm~\ref{alg:moment}); this recovery is performed via the standard linear programming approach.  Given this recovered distribution, $p^+$, to obtain the vector of estimated population eigenvalues (the spectrum), one simply returns the length $d$ vectors consisting of the $(d+1)$st-quantiles of distribution $p^+$---specifically, this is the vector whose $i$th component is the minimum value, $x$, with the property that the cumulative distribution function of $p^+$ at $x$ is at least $i/(d+1)$.   These two steps are formalized below:

\begin{algorithm}[H]
\begin{algorithmic}
\STATE \textbf{Input:} Approximation to first $k$ moments of population spectrum, $\bm{\hat{\alpha}}$, dimensionality $d$, and  fine mesh of values $\x = x_0,\ldots,x_t$ that cover the range $[0,b]$ where $b$ is an upper bound on the maximum population eigenvalue.  Taking $x_i = i \eps$ for $\eps \le 1/\max(d,n)$ is sufficient.  
\STATE \textbf{Output:} Estimated population spectrum, $\hat{\lambda_1}, \ldots, \hat{\lambda_d}.$
\vspace{.5cm}
\begin{enumerate}
\item Let $\p^+$ be the solution to the following linear program, which we will regard as a distribution consisting of point masses at values $\x$:
\begin{equation}
\begin{aligned}
\label{eqn:l1obj}& \underset{\p}{\text{minimize}} & & |\V\p-\bm{\hat{\alpha}}|_1 \\ 
& \text{subject to}                        & & \bm{1}^T\p=1\\
&                                                    & & \p>0,
\end{aligned}
\end{equation}
where the matrix $\V$ is defined to have entries $\V_{i,j} = x_j ^i.$
\item Return the vector $\hat{\lambda_1}, \ldots, \hat{\lambda_d}$ where $\hat{\lambda_i}$ is the $i$th $(d+1)$st-quantile of distribution corresponding to $\p^+$, namely set $\hat{\lambda_i} = \min \{x_j: \sum_{\ell \le j} p^+_{\ell} \ge \frac{i}{d+1}\}.$
\end{enumerate}

\end{algorithmic}
\caption{[Moments to Spectrum]) \label{alg:dis2vec}}
\end{algorithm}

The following restatement of Theorem~\ref{thm:main} quantifies the performance of the above algorithm.
\vspace{.5cm}

\noindent\textbf{Theorem~\ref{thm:main}.} \emph{
Consider an $n \times d$ data matrix $\Y = \X\S$, where $\X\in R^{n\times d}$ has i.i.d entries with mean $0$, variance $1$, and fourth moment $\fourthm$, and $\S$ is a real $d \times d$ matrix s.t. the eigenvalues of the population covariance $\Sigma= \S^T \S$, $\bm{\lambda} = \lambda_1,...\lambda_d$ are upper bounded by a constant $b\ge 1$. There is an algorithm that takes $\Y$ as input and for any integer $k\geq 1$ runs in time $poly(n,d,k)$ and outputs $\bm{\hat{\lambda}} = \hat{\lambda}_1,...\hat{\lambda}_d$ with expected $L_1$ error satisfying: 
$$
\E\left [\sum_{i=1}^d |\lambda_i-\hat{\lambda_i}|\right ] \le b d\left (f(k) (\frac{d^{k/2-1}}{n^{k/2}}+\frac{1+d^{\frac{1}{4}-\frac{1}{2k}}}{n^{1/2}}) +\frac{C}{k}+\frac{1}{d} \right ),
$$
 where $C$ is an absolute constant and $f(k)=C' (6k)^{3k+1}\fourthm^{k/2}$ for some absolute constant $C'.$   
}
\vspace{.5cm}

At the highest level, the proof of the above theorem has two main parts:  the first part argues that if two distributions have similar first $k$ moments, then the two distributions are ``close'' (in a sense that we will formalize soon).   As applied to our setting, the guarantees of Algorithm~\ref{alg:moment} ensures that, with high probability, the distribution returned by Algorithm~\ref{alg:dis2vec} will have similar first $k$ moments to the true population spectral distribution, and hence these two distributions are ``close''.  The second and straightforward part of the proof will then argue that if two distributions, $p$ and $p',$ are ``close'', and distribution $p$ consists of $d$ equally weighted point masses (such as the true population spectral distribution), then the vectors given by the $(d+1)$-quantiles of distribution $p'$ will be close, in $L_1$ distance, to the vector consisting of the locations of the point masses of distribution $p$.   As we show, the right notion of ``closeness'' of distributions to formalize the above proof approach is the \emph{Wasserstein} distance, also known as ``earthmover'' distance.

\begin{definition}\label{def:emd}
Given two real-valued distributions $p,q$, with respective density functions $p(x),q(x)$, the \emph{Wasserstein} distance between them, denoted by $W_1(p,q)$ is defined to be the cost of the minimum cost scheme of moving the probability mass in distribution $p$ to make distribution $q$, where the per-unit-mass cost of moving probability mass from value $a$ to value $b$ is $|a-b|.$   

One can also define Wasserstein distance via a dual formulation (given by the Kantorovich-Rubinstein theorem~\cite{kant} which yields exactly what one would expect from linear programming duality):  $$W_1(p,q) = \sup_{f: Lip(f)\leq 1}\int f(x)\cdot (p(x)-q(x)) dx,$$
where the supremum is taken over all functions with Lipschitz constant 1.
\end{definition}

Two convenient properties of Wasserstein distance are summarized in the following easily verified facts.  The first states that in the case of distributions consisting of $d$ equally-weighted point masses, the Wasserstein distance \emph{exactly} equals to the $L_1$ distance between the sorted vectors of the locations of the point masses.  The second fact states that given any distribution $p$, supported on a subset of the interval $[a,b]$, the distribution $p'$ defined to place weight $1/d$ at each of the $d$ $(d+1)$st-quantiles of distribution $p$, will satisfy $W_1(p,p') \le \frac{b-a}{d}.$  For our purposes, these two facts establish that, provided the distribution $\p^+$ returned by the linear programming portion of Algorithm~\ref{alg:dis2vec} is close to the true population spectral distance, in Wasserstein distance, then the final step of the algorithm---the rounding of the distribution to the point masses at the quantiles---will yield a close $L_1$ approximation to the vector of the population spectrum.

\begin{fact}\label{fact:emdvec}
Given two vectors $\vec{a}=(a_1,\ldots,a_d),$ and $\vec{b}=(b_1,\ldots,b_d)$ that have been sorted, i.e. for all $i$, $a_i \le a_{i+1}$ and $b_i \le b_{i+1}$, $$|\vec{a}-\vec{b}|_1 = d \cdot W_1(p_{\vec{a}},p_{\vec{b}}),$$ where $p_{\vec{a}}$ denotes the distribution that puts probability mass $1/d$ on each value $a_i$, and $p_{\vec{b}}$ is defined analogously.
\end{fact}

\begin{fact}\label{fact:emdquant}
Given a distribution $p$ supported on $[a,b]$,  let distribution $p'$ be defined to have probability mass $1/d$ at each of the $d$ $(d+1)$st-quantiles of distribution $p$.  Then $W_1(p,p') \le \frac{b-a}{d}.$
\end{fact}

Additionally, given any distribution $p$, supported on a subset of the interval $[a,b]$, the distribution $p'$ defined to place weight $1/d$ at each of the $d$ $(d+1)$st-quantiles of distribution $p$, will satisfy $W_1(p,p') \le \frac{b-a}{d}.$  Hence, for our purposes, as long as the distribution $\p^+$ returned by the linear programming portion of Algorithm~\ref{alg:dis2vec} is close to the true population spectral distance, in Wasserstein distance, then the final step of the algorithm---the rounding of the distribution to the point masses at the quantiles---will yield a close $L_1$ approximation to the vector of the population spectrum.

\medskip

The remaining component of our proof of Theorem~\ref{thm:main} is to establish that the accurate moment recover of Algorithm~\ref{alg:moment} as guaranteed by Theorem~\ref{thm:estmom} is sufficient to guarantee that, with high probability, the distribution $\p^+$ returned by the linear program in the first step of Algorithm~\ref{alg:dis2vec} is close, in Wasserstein distance, to the population spectral distribution.   We establish this general robust connection between accurate moment estimation, and accurate distribution recovery in Wasserstein distance, via the following proposition, which we prove in Section~\ref{sec:emd}.

\vspace{.5cm}
\noindent \textbf{Proposition~\ref{thm:redb}.} \emph{
Given two distributions with respective density functions $p, q$ supported on $[-1,1]$ whose first $k$ moments are $\bm{\alpha} = (\alpha_1,...\alpha_k)$ and $\bm{\beta} = (\beta_1,...\beta_k)$, respectively, the Wasserstein distance, $W_1(p,q)$, between $p$ and $q$ is bounded by: 
$$
W_1(p,q) \le\frac{C}{k}+g(k)\|\bm{\alpha}-\bm{\beta}\|_2, 
$$ where $C$ is an absolute constant, and $g(k)=C' 3^k$ for an absolute constant $C'$.}

\medskip

We conclude this section by assembling the pieces---the accurate moment estimation of Theorem~\ref{thm:estmom}, and the guarantees of the above proposition and Facts~\ref{fact:emdvec} and~\ref{fact:emdquant}---to prove Theorem~\ref{thm:main}.

\vspace{.5cm}
\noindent \emph{Proof of Theorem~\ref{thm:main}.}
Given the data matrix $Y$ and an upperbound on the population eigenvalues, $b$, our algorithm first divides each sample by $\sqrt{b}$ thereby reducing the problem to the setting where the eigenvalues are bounded by 1.  We then run Algorithm~\ref{alg:moment} on the scaled data matrix to recover the (scaled) moments.  Given these recovered moments, we then run Algorithm~\ref{alg:dis2vec} to recover the scaled spectrum, and then return this recovered spectrum scaled by the factor of $b$.   We now prove the correctness of this algorithm.

Let $\bm{\lambda}$  denote the vector of population eigenvalues, and let $\p$ denote the scaled spectral distribution obtained by dividing each entry of $\bm{\lambda}$ by $b$.  Hence $\p$ is support on $[0,1]$.  Denote the vector of the first $k$ moments of this distribution as $\balpha$.   Consider the distribution $\p^+$ recovered by the linear programming step of Algorithm~\ref{alg:dis2vec}, and denote its first $k$ moments with the vector $\balpha^+$.   We first argue that $\|\balpha^+-\balpha\|_2$ is small, and then apply the Wasserstein distance bound of Proposition~\ref{thm:redb}.  

By Proposition~\ref{variancebound} and the inequality $\E[X]^2\le \E[X^2]$, the estimated moment vector $\hat{\balpha}$, given as input to Algorithm~\ref{alg:dis2vec}, satisfies $\E \left[\|\balpha - \hat{\balpha}\|_1\right] \le \sum_{i=1}^k f(i) \max( \frac{d^{i/2-1}}{n^{i/2}}, \frac{d^{\frac{1}{4}-\frac{1}{2i}}}{\sqrt{n}},\frac{1}{\sqrt{n}})$ where $f(k)=2^{6k}k^{3k}\fourthm^{k/2}$. We now argue that there is a distribution that is a feasible point for the linear program that also has accurate moments.  Specifically, consider taking the mesh $\x = x_1,\ldots,x_t$ of the linear program grid points to be an $\epsilon$-mesh with $\epsilon\le \frac{1}{\max(n,d)}$, and consider the feasible point of the linear program that corresponds to the true scaled population spectral distribution, $\p$, whose support has been rounded to the nearest multiple of $\epsilon$.  This rounding changes the $i$th moment by at most $1-(1-\epsilon)^i$.  
 
 Hence, by the triangle inequality, this rounded population spectral distribution is a feasible point of the linear program, $\p^*,$ with objective value at most $\|\balpha-\hat{\balpha}\|_1 +\sum_{i=1}^k \left( 1-(1-\epsilon)^i \right).$  Hence, also by the triangle inequality, the moments $\balpha^+$ of the distribution  $\p^+$ returned by the linear program will satisfy 
\begin{align*}
\E[\|\balpha^+ - \balpha\|_2] \le \E[\|\balpha^+ - \balpha\|_1] \le 2 \sum_{i=1}^k \left( f(i) \max( \frac{d^{i/2-1}}{n^{i/2}}, \frac{d^{\frac{1}{4}-\frac{1}{2i}}}{\sqrt{n}},\frac{1}{\sqrt{n}})+ 1-(1-\epsilon)^i \right)\\
 \leq 2k \left(f(k) (\frac{d^{k/2-1}}{n^{k/2}}+\frac{1+d^{\frac{1}{4}-\frac{1}{2k}}}{n^{1/2}}) + k \epsilon \right),
 \end{align*}
 
Letting $\p^+_{quant}$ denote the distribution $\p^+$ that has been quantized so as to consist of $d$ equally weighted point masses (according to the second step of Algorithm~\ref{alg:dis2vec}), by Fact~\ref{fact:emdquant} and Proposition~\ref{thm:redb}, we have the following:
\begin{eqnarray*}
W_1(\p^+_{quant},\p) & \le & W_1(\p^+_{quant},\p^+) + W_1(\p^+,\p)\\
& \le & \frac{1}{d}+\left(\frac{C}{k}+g(k) ||\balpha^+ - \balpha||_2\right)\label{eqn:outerr}.
\end{eqnarray*}
Plugging in $\epsilon\le 1/\max(n,d)$ and our bound on the moment discrepancy $||\balpha^+ - \balpha||_2$ we get
$$
W_1(\p^+_{quant},\p) \le \frac{1}{d} + \frac{C}{k} + \widetilde{f}(k) (\frac{d^{k/2-1}}{n^{k/2}}+\frac{1+d^{\frac{1}{4}-\frac{1}{2k}}}{n^{1/2}}),
$$
where $\widetilde{f}(k) = C' (6k)^{3k+1}\fourthm^{k/2}$. Let $\bm{\hat{\lambda}}$ be the vector corresponds to distribution $\p^+_{quant}$ after multiplication by $b$. By Fact~\ref{fact:emdvec} we have 
$$\E[\sum_{i=1}^d |\lambda_i-\hat{\lambda_i}|] \le b d\left (\widetilde{f}(k) (\frac{d^{k/2-1}}{n^{k/2}}+\frac{1+d^{\frac{1}{4}-\frac{1}{2k}}}{n^{1/2}}) +\frac{C}{k}+\frac{1}{d} \right )$$. \qed

\medskip

To yield Corollary~\ref{cor:consistency} from Theorem~\ref{thm:main}, it suffices to show that, under the assumptions of the corollary, in the limit as $d\rightarrow \infty$, the number of moments that we can accurately estimate with our sublinear sample size, $n_d$, also goes to infinity (as $d\rightarrow \infty$).   By assumption, $\frac{n_d}{d^{1-\eps} }\rightarrow \infty$ for every constant $\eps>0$, and hence there is some function $\alpha(d)$ such that $\frac{n_d}{d^{1-\alpha(d)}}\rightarrow \infty$  with $\alpha(d) \rightarrow 0$; additionally we may assume that $\alpha(d) \ge \frac{1}{\log \log d}$.   By setting $k_d = \lfloor \frac{1}{\alpha(d)} \rfloor,$ from Theorem~\ref{thm:main}, we exam the expected Wasserstein error of our reconstruction term by term. The first term satisfies:
\begin{align*}
\widetilde{f}(k_d) \frac{d^{k_d/2-1}}{n^{k_d2}} = \left((c k_d)^{3k_d+1} \right)\frac{d^{k_d/2-1}}{n_d^{k_d/2}} \le \left((c k_d)^{3k_d+1} \right) \frac{d^{\frac{k_d}{2}-1}}{d^{\frac{k_d}{2}-\frac{k_d\alpha(d)}{2}}}\\
  \le  \left((c k_d)^{3k_d+1} \right)d^{-1/2} \le \frac{(c \log \log d)^{1+3 \log \log d} }{\sqrt{d}},
  \end{align*}
which tends to $0$ as $d \rightarrow \infty$. The second term satisfies:
$$
\widetilde{f}(k) \frac{1+d^{\frac{1}{4}-\frac{1}{2k}}}{n^{1/2}}\le (c k_d)^{3k_d+1} \frac{1+d^{\frac{1}{4}-\frac{1}{2k_d}}}{d^{\frac{1}{2}-\frac{\alpha(d)}{2}}} \le (c k_d)^{3k_d+1}\frac{1}{d^{1/4}}
$$
which tends to $0$ as $d \rightarrow \infty$. $C/k_d$ and $1/d$ also goes to $0$ as $d\rightarrow \infty$. Combining the four terms yields Corollary~\ref{cor:consistency}.

\section{Moments and Wasserstein Distance}~\label{sec:emd}
In this section we prove Proposition~\ref{thm:redb}, which establishes a general robust relationship between the disparity between the low-order moments of two univariate distributions, and the Wasserstein distance (see Definition~\ref{def:emd}) between the distributions.  This relatively straightforward proof proceeds via a constructive version of Jackson Theorem (see e.g. Theorem 7.4 in~\cite{carothers1998short}) which shows that Lipschitz functions can be well approximated by polynomials.  For convenience, we restate Proposition~\ref{thm:redb}, and the lower bound establishing its tightness:

\vspace{.5cm}
\noindent \textbf{Proposition~\ref{thm:redb}.} \emph{
Given two distributions with respective density functions $p, q$ supported on $[-1,1]$ whose first $k$ moments are $\bm{\alpha} = (\alpha_1,...\alpha_k)$ and $\bm{\beta} = (\beta_1,...\beta_k)$, respectively, the Wasserstein distance, $W_1(p,q)$, between $p$ and $q$ is bounded by: 
$$
W_1(p,q) \le \frac{C}{k}+g(k)\|\bm{\alpha}-\bm{\beta}\|_2, 
$$ where $C$ is an absolute constant, and $g(k)=C' 3^k$ for an absolute constant $C'$.}

The following lower bound shows that the inverse linear dependence in the above bound on the number of matching moments, $k$, is tight in the case where the moments exactly match:

\vspace{.5cm}
\noindent \textbf{Proposition~\ref{thm:emLowerBound}.} \emph{
For any even $k$, there exits a pair of distributions $p$, $q$, each consisting of $k/2$ point masses, supported within the unit interval $[-1,1]$, s.t. $p$ and $q$ have identical first $k-2$ moments, and Wasserstein distance $W_1(p,q) > 1/2k$.}

\subsection{Proof of Proposition~\ref{thm:redb}}\label{sec:proofap}
For clarity, we give an intuitive overview of the proof of Proposition~\ref{thm:redb} in the case where the first $k$ moments of the two distributions in question match exactly.   Consider a pair of distributions, $p$ and $q$, whose first $k$ moments match.   Because $p$ and $q$ have the same first $k$ moments, for \emph{any} polynomial $P$ of degree at most $k$,  the inner product between $P$ and $p-q$ is zero: $\int P(x)(p(x)-q(x))dx=0$.  The natural approach to bounding the Wasserstein distance, $\sup_{f \in Lip}\int f(x)\left(p(x)-q(x)\right)dx$, is to argue that for any Lipschitz function, $f$, there is a polynomial $P_f$ of degree at most $k$ that closely approximates $f$.   Indeed, \begin{eqnarray*} 
&\int f(x)(p(x)-q(x))dx \\
\le&\int |P_f(x)-f(x)|(p(x)-q(x))dx + \int P_f(x)(p(x) -q(x))dx \\
\le&2||f-P_f||_{\infty}.\label{eq:redb}
\end{eqnarray*}

Hence all that remains is to argue that there is a good degree $k$ polynomial approximation of any Lipschitz function $f$.  As the following standard fact shows, the approximation error of $f$ by a degree-$k$ polynomial is typically determined by the $k+1$th order derivative of $f$:

\begin{fact}[Polynomial Interpolation; e.g. Theorem 2.2.4 of \cite{de2012mathematics}]\label{lem:Interpolation}
For a given function $g\in C^{k+1}[a,b]$, there exists a degree $k$ polynomial $P_g$ such that 
$$
||g(x)-P_g(x)||_{\infty} \leq \left(\frac{b-a}{2}\right)^{k+1}\frac{\max_{x\in [a,b]}|g^{(k+1)}(x)|}{2^k(k+1)!}.
$$
\end{fact}

While our function $f$ is Lipschitz, its higher derivatives do not necessarily exist, or might be extremely large.  Hence before applying the above interpolation fact, we define a ``smooth'' version of $f$, which we denote $f_s$.  This smooth function will have the property that $||f-f_s||_{\infty}$ is small, and that the derivatives of $f_s$ are small.  We will accomplish this by defining $f_s$ to be the convolution of $f$ with a special ``bump'' function $\hat{b}$ that we define shortly.  To motivate our choice of $\hat{b},$ consider the convolution of $f$ with an arbitrary function, $h$:  $f_s = f * h$.  From the definition of convolution, the derivates of $f_s$ satisfy the following property:
$$
(f_s)^{(k+1)}(x)= (f * h^{(k+1)})(x).
$$
Hence we can bound the derivatives of $f_s$ by choosing $h$ with small derivatives.  Additionally, since we require that $f_s$ is close to $f$, we also want $h$ to be concentrated around $0$ so the convolution won't change $f$ too much in infinity norm.

We define $f_s$ to be the convolution of function $f$ with a scaled version of a special ``bump'' function $\hat{b}$ defined as the Fourier transform of the function $b(y)$ defined as $$b(y) = 
\begin{cases}
\exp(-\frac{y^2}{1-y^2})& |y|<1\\
0& \text{otherwise}
\end{cases}.$$

This function was leveraged in a recent paper by Kane et al.~\cite{kane2010exact}, to smooth the indicator function while maintaining small higher derivates.  As they show, the derivates of $\hat{b}$ are extremely well-behaved: $\|\hat{b}^{(k)}\|_1=O(\frac{1}{k})$ and $\|\hat{b}^{(k)}\|_\infty=O(1)$.   The actual function that we convolve $f$ with to obtain $f_s$ will be a scaled version of this bump function $\hat{b}_c =c\cdot \hat{b}(cx)$ for an appropriate choice of $c$.

We note that if, instead of convolving by $\hat{b}_c$, we had convolved by a scaled Gaussian (or a scaled version of the function $b(x)$ rather its Fourier transform) the $O(1/k)$ dependence on the Wasserstein distance that we show in Proposition~\ref{thm:redb} would, instead, be $O(1/\sqrt{k})$.\footnote{In the Gaussian case, this is because the $k$th derivative of a standard Gaussian, $G(x)$ is given by $G^{(k)}(x)=(-1)^kH_k(x)e^{-x^2}$ where $H_k(x)$ is the $k$th Hermite polynomial, and for even $k$ the value of $H_k(0)$ is $\frac{k!}{(k/2)!}$ which is already too large to obtain better than an $O(\frac{1}{\sqrt{k}})$ dependence.}

We now give the proof of Proposition~\ref{thm:redb} in the special case where the first $k$ moments of $p$ and $q$ match exactly.  The proof of the robust version is given in Appendix~\ref{ap:redb}, and is similar, though requires bounds on the coefficients of the interpolation polynomial $P$ that approximates the smoothed function $f_s$.

\vspace{.5cm}\noindent \emph{Proof of ``non-robust'' Proposition~\ref{thm:redb}.}
Consider distributions $p,q$ supported on the interval $[a,b]$ whose first $k$ moments match. 
Given a Lipschitz function $f$, let $f_s = f * \hat{b}_c(x)$ where the scaled bump function $\hat{b}_c$ is as defined above, for a choice of $c$ to be determined at the end of the proof.  Letting $P$ denote the degree $k$ polynomial approximation of $f_s$ we have the following:
$$\int f(x) \left(p(x)-q(x) \right)dx \le 2 ||f-P||_{\infty} \le 2||f-f_s||_{\infty} + 2||f_s-P||_{\infty}.$$

We bound each of these two terms.  For the first term, $||f-f_s||_{\infty}$, we have that for any $x$:
\begin{align*}
&|f(x)-f_s(x)| = |f(x) - \int_{-\infty}^{\infty} f(x-t) \hat{b}_c(t)dt| \\
=& |f(x)(1-\int_{-\infty}^\infty \hat{b}_c(t)dt) + \int_{-\infty}^\infty (f(x)-f(x-t)) \hat{b}_c(t)dt|
\leq \int_{-\infty}^\infty |\hat{b}_c(t)t| dt \label{eqn:bctt}
\end{align*}
Note that the last inequality holds since $\int \hat{b}_c(t)dt =b(0)=1$ and $f$ is has Lipschitz constant at most $1$, by assumption. To bound the above quantity, applying Lemma A.2 from \cite{kane2010exact} with $l=0,n=1$ yields $|\hat{b}_c(t)t| = O(1)$, with $l=0,n=3$ yields $|\hat{b}_c(t)t| = O(c^{-2} t^{-2})$. Splitting the integral into two parts, we have
\begin{align*}
\int_{-\infty}^{\infty} |\hat{b}_c(t)t| dt &\leq 2 \left(\int_{0}^{1/c} |\hat{b}_c(t)t| dt + \int_{1/c}^{\infty} |\hat{b}_c(t)t| dt\right) =O(\frac{1}{c}).
\end{align*}

We now bound the second term (the polynomial approximation error term) $||f_s-P||_{\infty}$, and then will specify the choice of $c$.  From Fact~\ref{lem:Interpolation}, this term is controlled by the $(k+1)$st derivative of $f_s$: 
$$
|(f_s)^{(k+1)}|_{\infty} = c^{k+1} |(f * (\hat{b}^{(k+1)})_c)(x)|_{\infty} \leq c^{(k+1)}|f|_{\infty}|\hat{b}^{(k+1)}|_{1}=O( c^{(k+1)}),
$$
where the inequality holds by the definition of convolution and the last equality applies Lemma A.3 from \cite{kane2010exact}. Hence we have the following bound on the polynomial approximation error term:
$$
||f_s-P||_{\infty} \le (\frac{b-a}{2})^{k+1}\frac{\max_{x\in [a,b]}|f_s^{(k+1)}(x)|}{2^k(k+1)!} = O\left(\frac{c^{k+1}}{2^k(k+1)!}\right).
$$
Setting $c = \Theta(k)$ balances the contribution from the two terms, $||f-f_s||_{\infty}$ and $||f_s-P||_{\infty}$, yielding the proposition in the non-robust case.
\qed

\subsection{Proof of Proposition~\ref{thm:emLowerBound}: Wasserstein Lower Bound}
We now prove Proposition~\ref{thm:emLowerBound}, showing that the $O(1/k)$ dependence of Proposition~\ref{thm:redb} is optimal up to constant factors, by constructing a sequence of distribution pairs $p_k, q_k$ with the same first $k$ moments but have $O(\frac{1}{k})$ Wasserstein distance between them.  The proof follows from leveraging a Chebyshev polynomial construction via the following general lemma:

\begin{lemma}\label{lemma:momentMatch}
Given a polynomial $P$ of degree $j$ with $j$ real roots $\{x_1,\ldots,x_j\}$, then letting $P'$ denote the derivative of $P$, then for all $\ell \le j-2$,  $\sum_{i=1}^{j} x_i^\ell \cdot \frac{1}{P'(x_i)} = 0.$
\end{lemma}
See Fact 14 in~\cite{valiant2011estimating} for the very short proof of the above lemma.

Lemma~\ref{lemma:momentMatch} provides a very natural construction for a pair of distributions whose low-order moments match: simply begin with any polynomial $P$ of degree $k$ with $k$ distinct (real) roots $x_1,\ldots,x_k$, and define the signed measure $m$, supported at the roots of $P$, with $m(x_i) = \frac{1}{P'(x_i)}$.  Define distribution $p^+$ to be the positive portion of $m$, normalized so as to be a distribution, and define $p^-$ to be the negative portion of $m$ scaled so as to be a distribution.  Note that provided $k \ge 2$, the scaling factor for $p^+$ and $p^-$ will be identical, as Lemma~\ref{lemma:momentMatch} guarantees that $\sum_{i=1}^{j} \frac{1}{P'(x_i)} = 0$, and hence the first $k-2$ moments of $p^+$ and $p^-$ will agree.

Proposition~\ref{thm:emLowerBound} will follow from setting the polynomial $P$ of the above construction to be the $k$th Chebyshev polynomial (of the first kind) $T_k$.  We require the following properties of the Chebyshev polynomials, which can be easily verified by leveraging the trigonometric definition of the Chebyshev polynomials: $T_k(\cos(t))=\cos(kt)$, and the fact that the derivative satisfies $T_k'(x) = k\cdot U_{k-1}(x)$ where $U_j$ is the $j$th Chebyshev polynomial of the second kind, satisfying $U_j(\cos(t)) = \frac{\sin\left((j+1)t\right)}{\sin t}.$

\begin{fact}
Let $x_1,\ldots,x_k$ denote the roots of $T_k$, with $x_i = -\cos \left( \frac{(1+2(i-1))\pi}{2k} \right),$ and set $y_i = 1/T_k'(x_i) = \frac{1}{k U_{k-1}(x_i)}$.
\begin{enumerate}
\item For $i \le n/2$, $\frac{i}{k^2} \le |y_i| \le i \frac{\pi}{k^2}$.
\item For $i \le n/2$, $\frac{5i}{k^2} \le |x_{i+1}-x_{i}| \le \frac{10i}{k^2}.$
\item $\sum_{i=1}^{n/2} y_{2i-1} = -\sum_{i=1}^{n/2} y_{2i} \in [\frac{1}{4},\frac{1}{2}].$ (Hence the scaling factor required to make the distributions from the signed measure is at least $2$.)
\end{enumerate}
\end{fact}

We now put the pieces together to complete the proof of Proposition~\ref{thm:emLowerBound}.

\begin{proof}[Proof of Proposition~\ref{thm:emLowerBound}]
By construction, and Lemma~\ref{lemma:momentMatch}, letting $p^+$ denote the distribution corresponding to the positive portion of the signed measure corresponding to $T_k$, and $p^-$ corresponding to the negative portion of the signed measure, we have that $p^+$ and $p^-$ each consist of $k/2$ point masses, located at values in the interval $[-1, 1]$, and the first $k-2$ moments of $p^+$ and $p^-$ are identical.   

To lower bound the Wasserstein distance between $p^+$ and $p^-$, note that all the mass in $p^+$ must be moved to the support of $p^-$.  Hence the distance is lower bounded by the sum: $$\sum_{i=1}^{k/4} 2 y_{2i} |x_{2i}-x_{2i - 1}| \ge \sum_{i=1}^{k/4}2\frac{2i}{k^2} \cdot \frac{10 i}{k^2} = \frac{40}{k^4} \sum_{i=1}^{k/4}i^2 \ge \frac{40}{64k}.$$
\end{proof}

\section{Empirical Performance}\label{sec:simulation}

We evaluated the performance of our population spectrum recovery algorithm on a variety of synthetic distributions, for a range of dimensions and sample sizes.  Recall that our algorithm consists of first applying Algorithm~\ref{alg:moment} to estimate the first $k$ moments of the population spectral distribution, and then applying Algorithm~\ref{alg:dis2vec} to recover a distribution whose moments closely match the estimated moments.  Our matlab implementation is available from our websites.

\subsection{Implementation Discussion}
\vspace{-.2cm}Our estimates of higher-order spectral moments have larger variance than our estimates of the lower-order spectral moments, hence when solving the moment inverse problem, we should be more forgiving of discrepancies in higher moments.  For example, we should require that the distribution we return match the estimated 1st and 2nd moments extremely accurately, while tolerating larger discrepancies between the 5th moment of the distribution that we return and the estimated 5th population spectral moment.   We implemented this intuition as follows: in the linear program of Algorithm~\ref{alg:dis2vec} that reconstructs a distribution from the moment estimates, the objective function of the linear program weighs the discrepancy between the $i$th moment of the returned distribution and the $i$th estimated moment by a coefficient $1/(c_i \hat{\alpha_i})$, where $\hat{\alpha_i}$ is the $i$th estimated moment, and $c_i$ is a scaling factor designed to capture the (multiplicative) standard deviation of the estimate.  In our experiments, we set $c_i$ to correspond to our bound on the standard deviation of the error in the $i$th recovered moment, implied by Theorem~\ref{thm:estmom}.   This corresponds to setting $c_i = (2i)^{2i} \cdot \frac{ \max(d^{i/2-1},1)}{n^{i/2}}.$   This scaling is theoretically justified, and we made no effort to optimize it: it seems likely that the empirical performance can be improved with a more careful weighting function.

In all runs of our algorithm, we estimated and matched the first $7$ spectral moments (i.e. we set the parameter $k=7$ in Algorithm~\ref{alg:dis2vec}). Considering higher moments beyond the $7$th did not significantly improve the results for the dimension and sample sizes that we considered.  

We would expect that the empirical performance could be improved by adaptively setting the number of moments to consider,  based on the values of the lower order moments.  Specifically, it would be natural to only consider higher moments if the lower order moments fail to robustly characterize the distribution.  More generally, a variety of other approaches to the general moment inverse problem could be substituted in place of Algorithm~\ref{alg:dis2vec}, and might improve the empirical performance, though such directions are beyond the focus of this work. 

\vspace{-.2cm}\subsection{Experimental Setup and Results}
\vspace{-.2cm}We evaluated our algorithm on four different types of population spectral distributions:
\begin{enumerate}
\vspace{-.1cm}\item Identity covariance: $\Sigma_d = \I_d$. (Figure~\ref{fig2}.)
\vspace{-.15cm}\item ``Two spike'' spectrum: $\Sigma_d$ has $d/2$ eigenvalues equal to $1$ and $d/2$ eigenvalues equal to $2$.  (Figure~\ref{fig3}.)
\vspace{-.15cm}\item Uniform spectrum: the eigenvalues of $\Sigma_d$ are $\{2/d,4/d,6/d,\ldots,2\}$, corresponding to a (discretized) uniform distribution over the range $[0,2]$.  (Figure~\ref{fig4}.)
\vspace{-.15cm}\item Toeplitz\footnote{Toeplitz matrices arise in numerous application areas, particularly in settings where each datapoint is a time-series and the correlation between two measurements decreases exponentially as a function of the chronological separation between the measurements.} covariance: $\Sigma_d(i,j) = 0.3^{|i-j|}$.   (Figure~\ref{fig5}.)
\end{enumerate}

\vspace{-.05cm}For each of the four types of population spectral distributions, we evaluated our algorithm for a variety of dimensions and sample sizes, taking $d = 512, 1024, 2048, 4096$, and for each value of $d$, we considered sample sizes $n = d/8, d/4, d/2, d, 2d.$   For each setting, we ran our algorithm five times on independently drawn data.   Figures~\ref{fig2}--\ref{fig5} show the results of each run, showing the cdf of the estimated spectral distribution (red), together with the cdf of the population spectral distribution (blue), and the cdf of the empirical spectral distribution (cyan).    

We observe that in general, for a fixed ratio of $d/n$, the results improve with larger $d$, as is implied by our theoretical sublinear sample size asymptotic consistency results (in spite of the daunting constant factors that appear in the analysis).  Additionally, our approach has good performance for the more difficult distributions---the uniform and Toeplitz distributions---even in the $n\le d$ regime. 

\newpage
\begin{figure}[H]
 \centering 
\vspace{-3.7cm}  \includegraphics[width=0.75\textwidth]{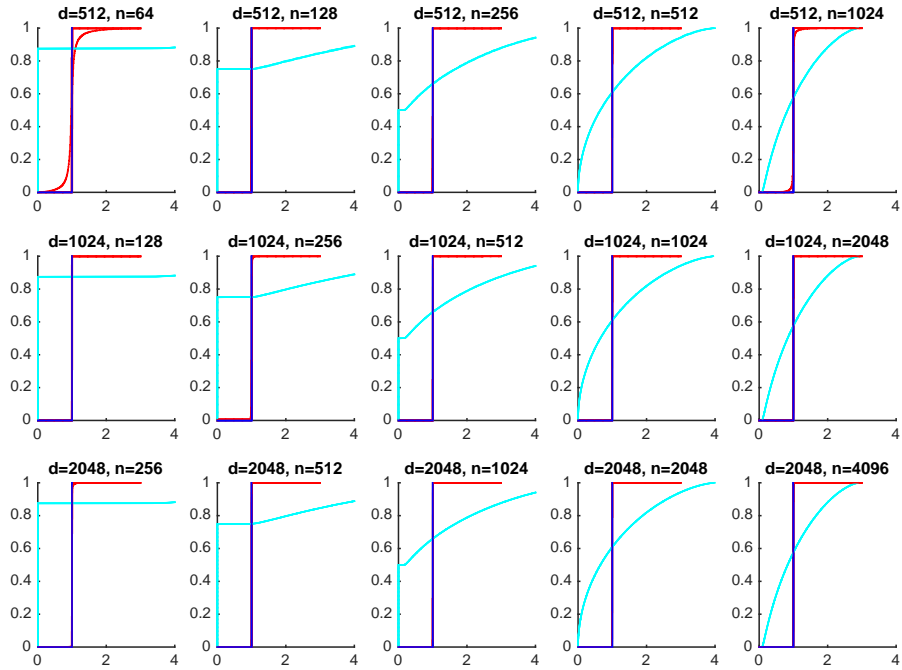}
\vspace{-3.7cm}\caption{\small{Empirical results for reconstructing the population spectrum for covariance $\Sigma = I_d$ for a range of sample sizes and dimensions.  Red lines depict the cdf of the distribution recovered by our algorithm over five independent trials, the blue line depicts the cdf of the true population spectral distribution, and the cyan line depicts the cdf of the empirical spectral distribution (in one of the trials).}\label{fig2}}
\end{figure}

\begin{figure}[H]
 \centering 
\vspace{-3.6cm} \includegraphics[width=0.75\textwidth]{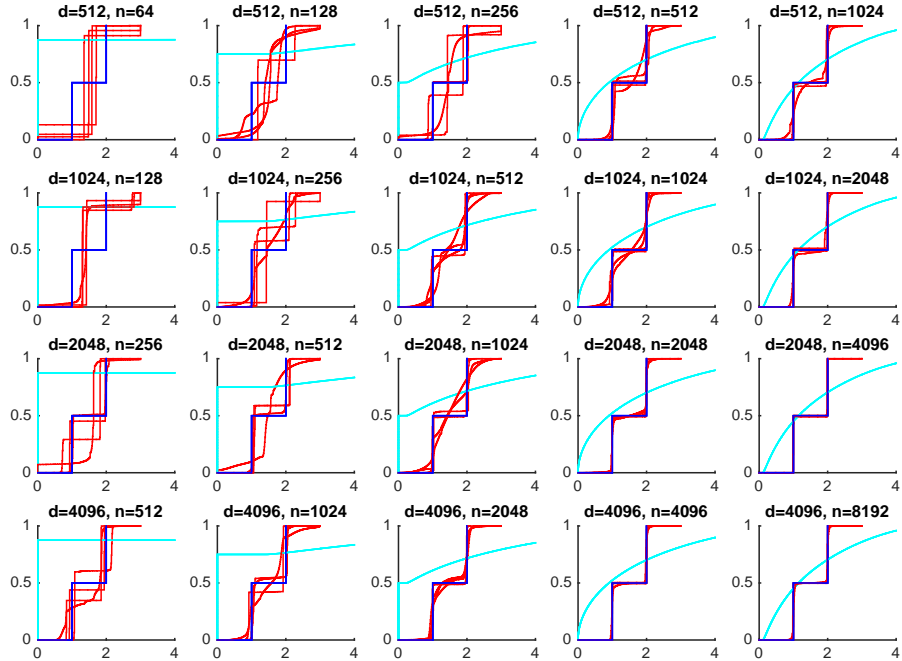}
\vspace{-3.7cm}\caption{\small{Empirical results for reconstructing the population spectrum for covariance matrices that have $d/2$ eigenvalues equal to $1$, and $d/2$ eigenvalues equal to  $2$.  Red lines depict the cdf of the distribution recovered by our algorithm over five independent trials, the blue line depicts the cdf of the true population spectral distribution, and the cyan line depicts the cdf of the empirical spectral distribution (in one of the trials).}\label{fig3}}
\end{figure}
\begin{figure}[H]
 \centering 
\vspace{-3.7cm}  \includegraphics[width=0.75\textwidth]{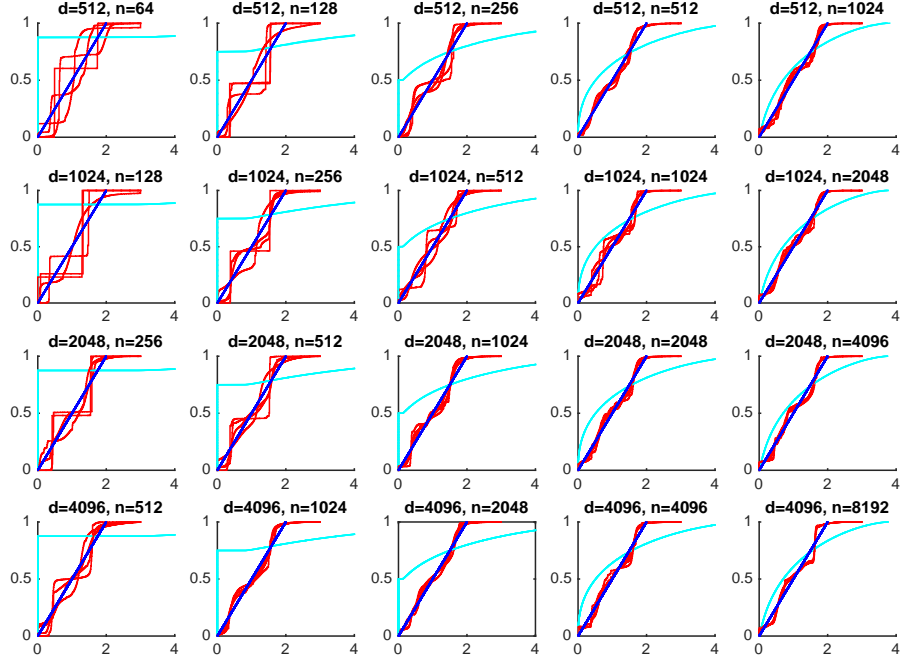}
\vspace{-3.7cm}\caption{\small{Empirical results for reconstructing the population spectrum for a covariance matrices whose eigenvalues correspond to the discretized uniform distribution on the interval $[0,2].$ Red lines depict the cdf of the distribution recovered by our algorithm over five independent trials, the blue line depicts the cdf of the true population spectral distribution, and the cyan line depicts the cdf of the empirical spectral distribution (in one of the trials).}\label{fig4}}
\end{figure}

\begin{figure}[H]
 \centering 
\vspace{-3.6cm}  \includegraphics[width=0.75\textwidth]{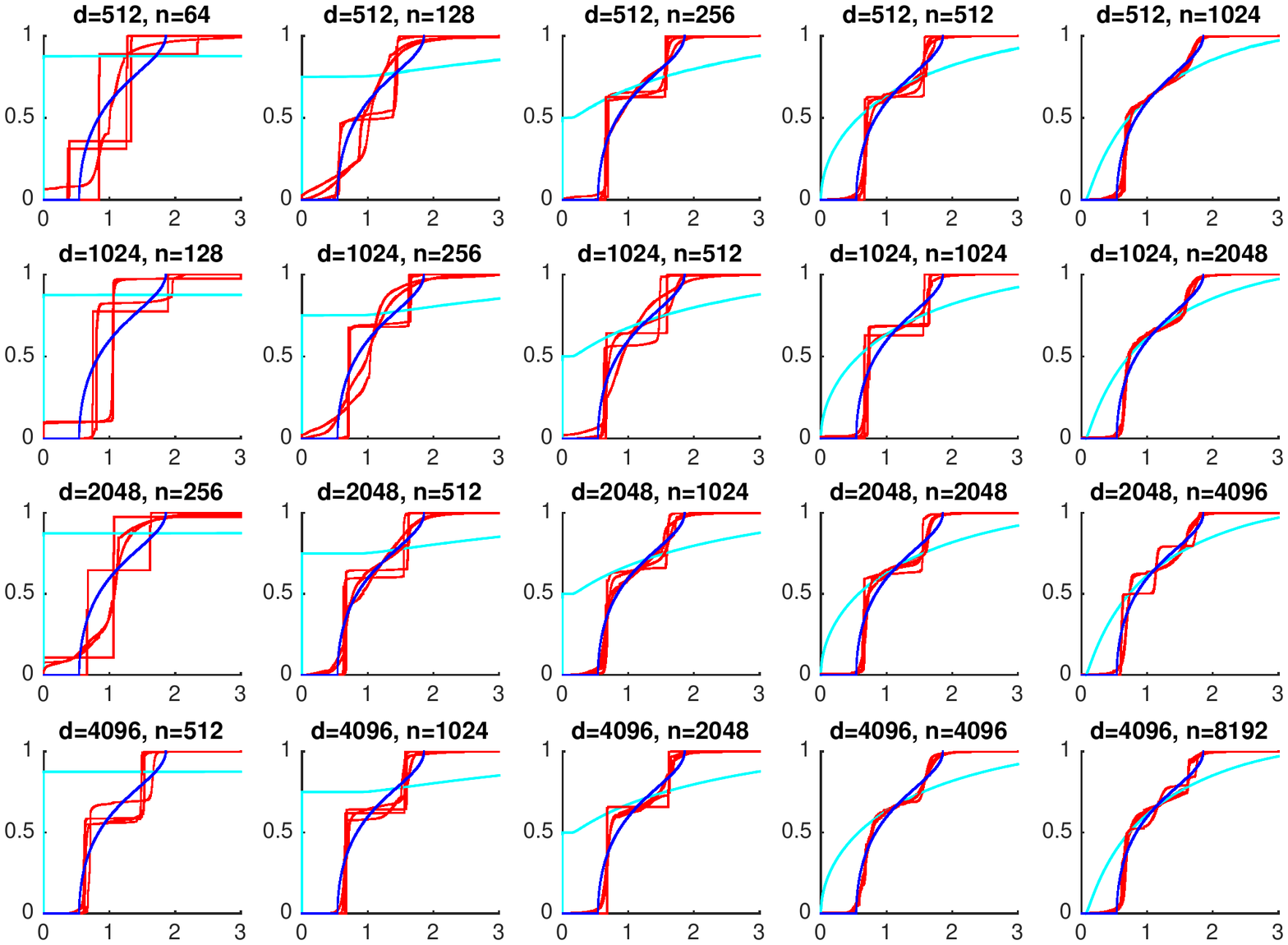}
\vspace{-3.7cm}\caption{\small{Empirical results for reconstructing the population spectrum for covariance $\Sigma = T$ where $T_{i,j}=0.3^{|i-j|}$ is a $d \times d$ Toeplitz matrix.  Red lines depict the cdf of the distribution recovered by our algorithm over five independent trials, the blue line depicts the cdf of the true population spectral distribution, and the cyan line depicts the cdf of the empirical spectral distribution (in one of the trials).}\label{fig5}}
\end{figure}

\bibliographystyle{plain}
\bibliography{soda}
\newpage
\noindent \begin{Huge}\textbf{Supplementary Material}\end{Huge}
\appendix

\section{Proof of Proposition~\ref{thm:redb}}\label{ap:redb}

Here we prove Proposition~\ref{thm:redb}, restated below for convenience.  The proof follows the general approach outlined in Section~\ref{sec:proofap}, though requires bounds on the coefficients of the interpolating polynomial to establish the robust proposition.

\vspace{.5cm}
\noindent \textbf{Proposition~\ref{thm:redb}.} \emph{
Given two distributions with respective density functions $p, q$ supported on $[-1,1]$ whose first $k$ moments are $\bm{\alpha} = (\alpha_1,...\alpha_k)$ and $\bm{\beta} = (\beta_1,...\beta_k)$, respectively, the Wasserstein distance, $W_1(p,q)$, between $p$ and $q$ is bounded by: 
$$
W_1(p,q) \le \frac{C}{k}+g(k)\|\bm{\alpha}-\bm{\beta}\|_2, 
$$ where $C$ is an absolute constant, and $g(k)=C' 3^k$ for an absolute constant $C'$.}

\vspace{.5cm}
\begin{proof}
As in the proof of the nonrobust version of the proposition given in Section~\ref{sec:proofap}, for any Lipschitz function $f$, we will consider a smoothed version $f_s = f * \hat{b}_c$, and a degree $k$ polynomial approximation, $P$.  In analogy with Equation~\ref{eq:redb} in the proof of the nonrobust version of the proposition given in Section~\ref{sec:proofap}, we have the following bound on the integral of $f$ with $p-q$:
$$W_1(p,q) \leq 2 \|f-P_k\|_{\infty} + (\bm{\alpha}-\bm{\beta})\c^T,$$ where $\c=(c_1,c_2,...,c_k)$ is the vector of coefficients of polynomial $P$.

We use the same polynomial $P$ as constructed in the proof of the nonrobust case, which applies Fact~\ref{lem:Interpolation} to yield the polynomial approximation of the function $f_s = f * \hat{b}_c$ constructed via interpolation on the Chebyshev nodes.   Recall from the proof of the nonrobust case in Section~\ref{sec:proofap}, that the first term in the distance bound is bounded as $2\|f-P\|_{\infty} \le O\left(1/k\right).$ 

To bound the remaining term, it suffices to bound the 2-norm of the vector of coefficients, $\c$. Note that $\c = \V_{k+1}^{-1} \y$ where $\V_{k+1}$ is the $k+1$ by $k+1$ Vandermonde matrix with Chebyshev nodes and $\y$ is the value of $f_s$ on these locations. Since $f_s$ is bounded by $O(1)$ (as, without loss of generality, $f$ is assumed to take values in the range $[0,2]$ and $\|f_s-f\|_\infty\leq O(\frac{1}{k})$), we can assume $\|\bm{y}\|_2\leq O(\sqrt{k})$ without loss of generality. It is known that the largest singular value of $\V_{k+1}^{-1}$ is  bounded by 
\begin{align}
\sqrt{\frac{2 \sum_{i=0}^k \sum_{j=0}^i a_{ij}^2}{k+1}},
\end{align}
where $a_{ij}$ is the degree $j$ coefficient of the $i$'th Chebyshev polynomial (see~\cite{li2008vandermonde}, for example). In order to bound the sum of squares of these coefficients of the Chebyshev polynomial, note the recurrence relation for Chebyshev polynomials given by  $
T_{n+1}(x) = 2xT_n(x)-T_{n-1}(x),$ with $T_0(x) =1,$ and $T_1(x) =x.$    Hence for the $i$th polynomial, we can loosely bound the sum of the squares of the coefficients as $\sqrt{\sum_{j=0}^i a_{ij}^2} \leq 3^i$. Plugging this in to bound $\|V_{k+1}^{-1}\|$ yields:
$$
\sqrt{\frac{2 \sum_{i=0}^k \sum_{j=0}^i a_{ij}^2}{k+1}} \leq \sqrt{\frac{9^{k+1}-1}{4(k+1)}}.
$$
Thus $\|\bm{c}\|_2 \le  \|\V_{k+1}^{-1}\| \|\y\| \leq  O(\sqrt{k}\sqrt{\frac{9^{k+1}-1}{4(k+1)}}) \leq O( \frac{3^{k+1}}{2})$.  Setting $g(k)=O(\frac{3^{k+1}}{2})$ yields the claimed proposition.
\end{proof}

\section{Proof of Proposition~\ref{variancebound}}~\label{ap:varb}
In this section we prove Proposition~\ref{variancebound} to establish the bound on the variance of our moment estimator.  See the discussion at the end of Section~\ref{sec:estmomproof} for the high-level proof approach.    The following definitions will be useful for defining a partition of the set of increasing cycles.

\begin{definition}
Let $P=\{P_1,...,P_m\}$ denote a partition of variables $\{\sigma,\pi\}$ where each set $P_i$ has size at most $2$, and let $Q=\{Q_1,\ldots,Q_w\}$ denote a partition of variables $\{\delta,\gamma,\delta',\gamma'\}$.   We use the notational shorthand $P(x)$ to indicate the set in $P$ which contains variable $x$, and analogously for $Q(x).$ 
We say that $(\sigma,\pi)$ are \emph{subject} to $P$, if it holds that two variables in ${\sigma,\pi}$ have the same value if and only if they fall in the same class of $P$. 
\end{definition}

To illustrate, if $(\sigma,\pi)$ is subject to $P_1$ and $P_1 = \{\sigma_1,\pi_2\}$, then it must be the case that $\sigma_1=\pi_2$. 

\begin{definition}
We say that $(\delta,\gamma,\delta',\gamma')$ is subject to the partition $Q=(Q_1,\ldots,Q_w)$, if, for every pair of variables in the set ${\delta,\gamma,\delta',\gamma'}$ that are in the same partition, $Q_i$, they have the same value.  In contrast to the definition for $(\sigma,\pi)$ being subject to $P$, this is not an if, and only if: two variables in different partitions, $Q_i,Q_j$ can take the same value.
\end{definition}

\begin{definition}
We say that a partition $Q = \{Q_1,...,Q_w\}$ of the variables $\{\delta,\gamma,\delta',\gamma'\}$ \emph{respects} the partition $P=(P_1,\ldots,P_m)$ of the variables $\{\sigma,\pi\}$ if each partition $Q_i$ is contained in some set $P_j^* = \{ \delta_\ell,\gamma_\ell |\sigma_\ell \in P_j\}\cup\{ \delta'_\ell,\gamma'_\ell |\pi_\ell \in P_j\}$. 
\end{definition}

\begin{definition}
Given partitions $Q=(Q_1,\ldots,Q_w)$ and $P$, such that $Q$ respects $P$, we say that a list of variables $(\delta,\gamma,\delta',\gamma')$ is \emph{subject to} $(P,Q)$ if 
\begin{enumerate}
	\item $(\delta,\gamma,\delta',\gamma')$ is subject to $Q$,
	\item for every pair of variables $(x,y) \in P_j^* = \{ \delta_\ell,\gamma_\ell |\sigma_\ell \in P_j\}\cup\{ \delta'_\ell,\gamma'_\ell |\pi_\ell \in P_j\}$, then $x=y$ if and only if $x$ and $y$ are in the same partition of $Q$.
\end{enumerate}
\end{definition}
Throughout we will only refer to pairs $(P,Q)$ for which $Q$ respects $P$.

\begin{fact}
Given values of $\sigma,\pi,\delta,\gamma,\delta',\gamma'$, there is a unique pair $(P,Q)$ where $(\sigma,\pi)$ is subject to $P$, $(\delta,\gamma,\delta',\gamma')$ is subject to $(P,Q)$. We say that $\sigma,\pi,\delta,\gamma,\delta',\gamma'$ is subject to $(P,Q)$.
\end{fact}
\begin{proof}
$P$ can be uniquely determined by the values of $\sigma,\pi$, after which $Q$ can be determined by $\delta,\gamma,\delta',\gamma'$.
\end{proof}
\begin{fact}\label{fact:xPQ}
Given any $(\sigma,\pi,\delta,\gamma,\delta',\gamma')$ and $(\sigma^1,\pi^1,\delta^1,\gamma^1,\delta'^1,\gamma'^1)$ subject to $(P,Q)$, we have 
$$
\E[\prod_{i=1}^k \X_{\sigma_i \delta_{i}} \X_{\sigma_i \gamma_{i}} \X_{\pi_i \delta'_{i}} \X_{\pi_i \gamma'_{i}}] = \E[\prod_{i=1}^k \X_{\sigma^1_i 
\delta^1_{i}} \X_{\sigma^1_i \gamma^1_{i}} \X_{\pi^1_i \delta'^1_{i}} \X_{\pi^1_i \gamma'^1_{i}}]
$$.
\end{fact}
The above fact follows from the observation that the value of $\E[\prod_{i=1}^k \X_{\sigma_i \delta_{i}} \X_{\sigma_i \gamma_{i}} \X_{\pi_i \delta'_{i}} \X_{\pi_i \gamma'_{i}}]$ is determined only by the pair $(P,Q)$ that $(\sigma,\pi,\delta,\gamma,\delta',\gamma')$ subjects to. With the help of Fact~\ref{fact:xPQ}, we are able to consider all the terms in a single $(P,Q)$ class together because they all have the same $\X$ ``part'' in expectation.

\begin{fact}
Given $\sigma,\pi,\delta,\gamma,\delta',\gamma'$ subject to $(P,Q)$, a necessary condition for 
$$
\E[\prod_{i=1}^k \X_{\sigma_i \delta_{i}} \X_{\sigma_i \gamma_{i}} \X_{\pi_i \delta'_{i}} \X_{\pi_i \gamma'_{i}}] - \E[\prod_{i=1}^k \X_{\sigma_i \delta_{i}} \X_{\sigma_i \gamma_{i}}] \E[\prod_{i=1}^k \X_{\pi_i \delta'_{i}} \X_{\pi_i \gamma'_{i}}]
$$
to be non-zero is that $|Q_i|$ is even for all $i\in \{1,\ldots,w\}$ and there exists $\pi_j\in P(\sigma_i)$ and either $\delta_j$ or $\gamma_j \in Q(\delta_i')$. We call these $(P,Q)$ \textbf{non-zero}.
\end{fact}
If for some $i$, $|Q_i|$ is not even, there exists some entry $X_{ij}$ whose number of occurrence in the product is $1$ and thus the expectation term will be zero. Otherwise if there does not exist $\pi_j\in P(\sigma_i)$ and $\delta_j$ or $\gamma_j \in Q(\delta_i)$, the product $\prod_{i=1}^k \X_{\sigma_i \delta_{i}} \X_{\sigma_i \gamma_{i}}$ will be independent of $\prod_{i=1}^k \X_{\pi_i \delta'_{i}} \X_{\pi_i \gamma'_{i}}$ which also yields zero for the above formula.

Given $(P,Q)$ and $(\sigma,\pi,\delta,\gamma,\delta',\gamma')$ subject to $(P,Q)$, we have the following two definitions:
\begin{definition}
The product $T_{x_1,y_1}T_{x_2,y_2}...T_{x_p,y_p}$ is called a \textbf{free cycle} of length $p$ if and only if for any $i$, variables $y_i\in P(x_{i+1})$ and $|P(x_{i+1})|=2$. Similarly, the product $T_{x_1,y_1}T_{x_2,y_2}...T_{x_p,y_p}$ is called an \textbf{arc} of length $p$ if and only if for $i$ from $1$ to $p-1$, $y_i\in P(x_{i+1})$ and $|P(x_{i+1})|=2$, $|P(x_1)|>2$ and $|P(y_p)|>2$. 
\end{definition}

\begin{lemma}\label{sumofQ}
Given $Q = \{Q_1,...,Q_w\}$ as a partition of variables $\delta,\delta',\gamma,\gamma'$ where each $Q_i$ has size either $2$ or $4$, 
$$
\sum_{(\delta,\gamma,\delta',\gamma')|Q}\prod_{i=1}^k \T_{\delta_i, \gamma_{i+1}} \T_{\delta'_{i},\gamma'_{i+1}} = \sum_{x_1,...,x_{\xi},y_1,...,y_{\xi}} \prod_{i=1}^{\xi} (T^{l_i})_{x_i,y_i} \prod_{j}^\eta tr(T^{p_j}) \leq \prod_{i=1}^{\xi}tr(T^{2l_i})^{1/2} \prod_{j}^\eta tr(T^{p_j})
$$
where $x_1,...,x_{\xi},y_1,...,y_{\xi}$ are aliases of variables in $(\delta,\gamma,\delta',\gamma')$ in a set of size $4$ of $Q$. $\prod_{i=1}^k \T_{\delta_i, \gamma_{i+1}} \T_{\delta'_{i},\gamma'_{i+1}}$ contains $\xi$ arcs of lengths $l_1,...,l_\xi$, $\eta$ free cycles of lengths $p_1,...,p_\eta$
\end{lemma}
\begin{proof}
The first equality follows from re-grouping the terms in the product.  The inequality is obtained by applying the following lemma (Lemma~\ref{cauchyschwarz}).
\end{proof}
\begin{lemma}\label{cauchyschwarz}
Let $ R = \sum_{a_1,...a_r} \prod_{s=1}^{\xi} g_s(a_{h_s},a_{t_s}) $ be a function of $a_1,\dots,a_r$, where each $a_i$ occurs at least four times. Then,
$R^2 \leq \prod_{s=1}^{\xi} \sum_{a_{h_s},a_{t_s}} g_s^2 (a_{h_s},a_{t_s})$.
\end{lemma}
\begin{proof}
We proceed by induction on $r$. Assume the inequality holds for $r = r^*-1$. Then for $r = r^*$ $R = \sum_{a_{r^*}}(\sum_{a_1,...a_{r^*-1}}\prod_{s=1}^{\xi} g_s(a_{h_s},a_{t_s}))$. We can use the inductive assumption on the sum inside the the parenthesis, where we treat $a_{r^*}$ as a constant,
$$
R\leq \bigg( \sum_{a_{r^*}} \prod_{s:h_s=r^*\vee r_s=r^*} \sqrt{\sum_{a_{h_s}:h_s\neq r^*} \sum_{a_{t_s}:t_s\neq r^*} g_s^2 (a_{h_s},a_{t_s})}\bigg) \prod_{s:h_s\neq r^*\land t_s\neq r^*} \sqrt{\sum_{a_{h_s},a_{t_s}} g_s^2 (a_{h_s},a_{t_s})}   
$$
Since $r^*$ occurs more than $2$ times, we can use the Cauchy-Schwarz inequality to bound the summation in the parenthesis. Assume $g_{i_1},g_{i_2},\ldots,g_{i_\xi}$ are the functions in the first product. Also for simplicity we will assume $h_{i_1}=r^*$ and $t_{i_1}\neq r^*$. Applying Cauchy-Schwarz yields:
\begin{align*}
R^2 & \leq \bigg(\sum_{a_{r^*},a_{t_{i_1}}} g_{i_1}^2 (a_{r^*},a_{t_{i_1}})\bigg) \bigg(\sum_{a_{r^*}}\Big( \prod_{j=2}^\xi \sum_{a_{h_{i_j}}:h_{i_j}\neq r^*} \sum_{a_{t_{i_j}}:t_{i_j}\neq r^*} g_{i_j}^2 (a_{h_{i_j}},a_{t_{i_j}})\Big) \bigg)\prod_{s:h_s\neq r^*\land t_s\neq r^*} \sum_{a_{h_s},a_{t_s}} g_s^2 (a_{h_s},a_{t_s})    \\
    & \leq \bigg(\sum_{a_{r^*},a_{t_{i_1}}} g_{i_1}^2 (a_{r^*},a_{t_{i_1}})\bigg)\Big( \prod_{j=2}^\xi\sum_{a_{h_{i_j}},a_{t_{i_j}}} g_{i_j}^2 (a_{h_{i_j}},a_{t_{i_j}})\Big) \bigg)\prod_{s:h_s\neq r^*\land t_s\neq r^*} \sum_{a_{h_s},a_{t_s}} g_s^2 (a_{h_s},a_{t_s})\\
    & \leq \prod_{s=1}^{\xi} \sum_{a_{h_s},a_{t_s}} g_s^2 (a_{h_s},a_{t_s})
\end{align*}
Where in the second inequality we have utilized the fact that $g_{i_j}^2$ is always non-negative and hence the sum of the product is less than the product of the sum.  
\end{proof}


\begin{lemma}\label{partitionsize}
Given $(P=\{P_1,...,P_m\}$,$Q = \{Q_1,...,Q_w\})$ which is non-zero, let $(\sigma,\pi)$ be subject to $P$, and $(\delta,\gamma,\delta',\gamma')$ subject to $Q$. Suppose $\prod_{i=1}^k \T_{\delta_i, \gamma_{i+1}} \T_{\delta'_{i},\gamma'_{i+1}}$ contains $\xi$ arcs of lengths $l_1,...,l_\xi$, $\eta$ free cycles of lengths $p_1,...,p_\eta$, the following bound of $\xi/2+\eta$ holds
$$
\xi/2+\eta\leq 2k-m
$$
\end{lemma}
\begin{proof}
We proceed by induction on $k$.  Suppose the claim holds for $k \le \ell-1$ and consider $k=\ell$:
\begin{itemize}
\item If $m>\ell$, there will be at least one label $\sigma_a$ satisfying $P(\sigma_a)=\{\sigma_a\}, Q(\delta_a) = \{\delta_a,\gamma_a\}$ and one label $\pi_b$ satisfying $P(\pi_b) = \{\pi_b\}, Q(\delta'_a) = \{\delta'_a,\gamma'_a\}$, we can delete $\sigma_a, \delta_a,\gamma_a, \pi_b,\delta'_b,\gamma'_b$ and for $a<i\le \ell$ rename $\sigma_i,\delta_i,\gamma_i$ as $\sigma_{i-1},\delta_{i-1},\gamma_{i-1}$, for $b<j\le \ell$ rename $\pi_j,\delta'_j,\gamma'_j$ as $\sigma_{i-1},\delta_{i-1},\gamma_{i-1}$. It is easy to check that neither the number of  free cycles nor the number of arcs decreases.  This corresponds to the $\ell-1$ case and hence by the inductive assumption, $\xi/2+\eta\leq 2(\ell-1)-(m-2)=2\ell-m$. \item If $m=\ell$, each $P_i$ will be size $2$.
	\begin{itemize}
	  \item If there exists $Q_i=\{\delta_a,\gamma_a,\delta'_b,\gamma'_b\}$, each elements of $Q_i$ will be a head of some arc and since each arc has $2$ heads, deleting all elements of $Q_i$ can incur a decrease of $\xi$ of $2$, by the induction assumption we have $\xi/2+\eta -1 \leq 2(\ell-1)-(m-1)=2\ell-m-1$. 
	  \item Finally if each $Q_i$ has size $2$, in this case there are no arcs but only free cycles, pick $i,j$ such that $P(\sigma_i)=\{\sigma_i,\pi_j\}$ and $Q(\delta_i) = \{\delta_i,\delta'_j\}$ or $Q(\delta_i) = \{\delta_i,\gamma'_j\}$. For simplicity we assume $Q(\delta_i) = \{\delta_i,\delta'_j\}$, such $i,j$ must exist because $(P,Q)$ is non-zero. We do the same procedure as before to delete $\{\delta_i,\delta'_j, \gamma_i,\gamma'_j\}$ and rename the other terms. After doing this, the free cycle containing the original $\delta_i, \delta'_j$ will be connected with the free cycle containing $\gamma_i,\gamma'_j$; thus by the inductive assumption, $\xi/2+(\eta-1)\leq 2(\ell-1)-(m-1) = 2\ell-m-1$.
	\end{itemize}
\item Base case. If $k=1$, $P=(\sigma_1,\pi_1)$ is fixed, there are two possible cases for $Q$:
	\begin{itemize}
		\item $Q_1={\delta_1,\gamma_1,\delta'_1,\gamma'_1}$, in this case $\xi=2$ and $\xi/2+\eta\leq 2k-m$.
		\item $Q_1={\delta_1,\delta'_1},Q_2={\gamma_1,\gamma'_1}$, in this case $\eta=1$ and $\xi/2+\eta\leq 2k-m$.
	\end{itemize}
\end{itemize}
\end{proof}
\begin{lemma} \label{Qbound}
Given $P=\{P_1,\ldots,P_m\}$, for any $Q = \{Q_1,\ldots,Q_w\}$ such that $(P,Q)$ is non-zero, the sum of the $\T$ ``part'' over all $(\delta,\gamma,\delta',\gamma')$ respecting $Q$ satisfies the following bound:
$$
\sum_{(\delta,\gamma,\delta',\gamma')|Q}\prod_{i=1}^k \T_{\delta_i, \gamma_{i+1}} \T_{\delta'_{i},\gamma'_{i+1}} \leq \|T\|_{4m-4k+2}^{2m-2k+1} \|T\|_{2}^{4k-2m-1}
$$
\end{lemma}
\begin{proof}
By Lemma~\ref{sumofQ}, we have 
\begin{align}
\sum_{(\delta,\gamma,\delta',\gamma')|Q}\prod_{i=1}^k \T_{\delta_i, \gamma_{i+1}} \T_{\delta'_{i},\gamma'_{i+1}} \leq \prod_{i=1}^{\xi}tr(T^{2l_i})^{1/2} \prod_{j}^\eta tr(T^{p_j}) \label{eqn:sumofQ1}
\end{align}
where $\sum_{i=1}^\xi l_i + \sum_{i=1}^\eta p_i = 2k$. By the following inequalities:
\begin{align*}
\|T\|^a_{2a} \|T\|^b_{2b} &\geq \|T\|_{a+b}^{a+b},\\
\|T\|^a_{a} \|T\|^b_{b} &\geq \|T\|^{a-1}_{a-1} \|T\|^{b+1}_{b+1}, \qquad \text{provided }a>b\\
\|T\|^a_a\|T\|^b_b &\geq \|T\|_{a+b}^{a+b}
\end{align*}
and $\xi/2+\eta\leq 2k-m$ as shown in Lemma~\ref{partitionsize}, the product on the right hand side of Equation~\ref{eqn:sumofQ1} achieves its maximum when $\xi = 2(2k-m),\eta = 0$ and $l = (2m-2k+1,1,\ldots, 1)$. Thus we have $\prod_{i=1}^{\xi}tr(T^{2l_i})^{1/2} \prod_{j}^\eta tr(T^{p_j})\leq \|T\|_{4m-4k+2}^{2m-2k+1} \|T\|_{2}^{4k-2m-1}$, as desired.

\end{proof}
\begin{lemma}\label{PQbound}
Given $(P=\{P_1,\ldots,P_m\}$, $Q = \{Q_1,\ldots,Q_w\})$ which is non-zero,  
$$
\sum_{(\delta,\gamma,\delta',\gamma')|(P,Q)}\prod_{i=1}^k \T_{\delta_i, \gamma_{i+1}} \T_{\delta'_{i},\gamma'_{i+1}} \leq 2^{2k} \|T\|_{4m-4k+2}^{2m-2k+1} \|T\|_{2}^{4k-2m-1}
$$
\end{lemma}
\begin{proof}
For convenience, we will use the notation $(\delta,\gamma,\delta',\gamma')|(P,Q)$ as a shorthand for denoting that $(\delta,\gamma,\delta',\gamma')$ is subject to $P,Q$.  Note that the difference between the constraints $(\delta,\gamma,\delta',\gamma')|(P,Q)$ and $(\delta,\gamma,\delta',\gamma')|Q$ is that the first one contains inequality constraints while the second one only has equality constraints, in other words, when we say a set of indices is subject to $Q$ we only require some of the indices are equal to each other, while being ``subject to $P,Q$'' additionally requires some of these indices to be not equal. Denote these inequality constraints as $I_1,I_2,\ldots,I_h$, and consider applying  the  inclusion-exclusion principle:
\begin{align*}
\sum_{(\delta,\gamma,\delta',\gamma')|(P,Q)}\prod_{i=1}^k \T_{\delta_i, \gamma_{i+1}} \T_{\delta'_{i},\gamma'_{i+1}} = \sum_{(\delta,\gamma,\delta',\gamma')|Q}\prod_{i=1}^k \T_{\delta_i, \gamma_{i+1}} \T_{\delta'_{i},\gamma'_{i+1}} - \sum_{a} \sum_{(\delta,\gamma,\delta',\gamma')|Q,\bar{I_a}}\prod_{i=1}^k \T_{\delta_i, \gamma_{i+1}} \T_{\delta'_{i},\gamma'_{i+1}}\\
+\sum_{a<b}\sum_{(\delta,\gamma,\delta',\gamma')|Q,\bar{I_a},\bar{I_b}}\prod_{i=1}^k \T_{\delta_i, \gamma_{i+1}} \T_{\delta'_{i},\gamma'_{i+1}} -\sum_{a<b<c}\sum_{(\delta,\gamma,\delta',\gamma')|Q,\bar{I_a},\bar{I_b},\bar{I_c}}\prod_{i=1}^k \T_{\delta_i, \gamma_{i+1}} \T_{\delta'_{i},\gamma'_{i+1}}...
\end{align*}
Let $\bar{I}$ be the complement of constraint $I$.  The constraints of the form $(\delta,\gamma,\delta',\gamma')|Q,\bar{I_a},\bar{I_b}$ are equivalent to $(\delta,\gamma,\delta',\gamma')|Q'$ for some $Q$ where $Q'$ respects $P$. Note that if $(P,Q)$ is non-zero, so is $(P,Q')$ and hence Lemma~\ref{Qbound} can be applied for each term in above formula. The number of inequalities can be upper bounded by $2k$ as follows:  each inequality corresponds to a pair $Q_a,Q_b$ and a set $P^*_i$ satisfying $Q_a,Q_b\subset P^*_i$, hence the size of partition $P$ being at most $2k$ yields a $2k$ upper bound on the number of inequalities. The number of ``sum of product'' terms coming from the inclusion-exclusion principle equals to the number of subsets of the inequality sets which is upper bounded by $2^{2k}$, hence the claimed upper bound follows.
\end{proof}
\begin{proof}[Proof of Proposition~\ref{variancebound}]
For each non-zero $(P=\{P_1,...,P_m\}$,$Q = \{Q_1,...,Q_w\})$, its contribution to the variance is:
\begin{align*}
\frac{1}{|S|^2}\sum_{(\sigma,\pi,\delta,\gamma,\delta',\gamma')|(P,Q)} (\E[\prod_{i=1}^k \X_{\sigma_i \delta_{i}} \X_{\sigma_i \gamma_{i}} \X_{\pi_i \delta'_{i}} \X_{\pi_i \gamma'_{i}}] - \E[\prod_{i=1}^k \X_{\sigma_i \delta_{i}} \X_{\sigma_i \gamma_{i}}] \E[\prod_{i=1}^k \X_{\pi_i \delta'_{i}} \X_{\pi_i \gamma'_{i}}]) \prod_{i=1}^k \T_{\delta_i, \gamma_{i+1}} \T_{\delta'_{i},\gamma'_{i+1}}\\
\leq \frac{\fourthm^k}{|S|^2} \sum_{(\sigma,\pi,\delta,\gamma,\delta',\gamma')|(P,Q)} \prod_{i=1}^k \T_{\delta_i, \gamma_{i+1}} \T_{\delta'_{i},\gamma'_{i+1}},
\end{align*}
the above holds  due to the fact that $\fourthm$, the fourth moment of $\X_{i,j}$, is an upper bound for any $\E[\X_{\sigma_i \delta_{i}} \X_{\sigma_i \gamma_{i}} \X_{\pi_i \delta'_{i}} \X_{\pi_i \gamma'_{i}}],$
$$
\leq \fourthm^k\frac{n^m}{n^{2k}} \sum_{(\delta,\gamma,\delta',\gamma')|(P,Q)} \prod_{i=1}^k \T_{\delta_i, \gamma_{i+1}} \T_{\delta'_{i},\gamma'_{i+1}},
$$
where $|S|=\frac{n!}{(n-k)!}$ is approximated by $n^k$,
$$
\leq 2^{2k}\fourthm^k\frac{n^m}{n^{2k}} \|T\|_{4m-4k+2}^{2m-2k+1} \|T\|_{2}^{4k-2m-1},
$$
where the last inequality comes from Lemma~\ref{PQbound}. If $k=1$, we simply have $\frac{n^m}{n^{2k}} \|T\|_{4m-4k+2}^{2m-2k+1} \|T\|_{2}^{4k-2m-1} \leq \frac{1}{n} \|T\|_{k}^{2k}$. Otherwise by Holder's inequality, 
\begin{align*}
\|T\|_{2}&\leq d^{1/2-1/k}\|T\|_{k} \\
\|T\|_{4m-4k+2}^{2m-2k+1}&\leq d^{1/2-(2m-2k+1)/k}\|T\|_{k}^{2m-2k+1} \qquad \text{provided } 2(2m-2k+1)\leq k
\end{align*} 
Thus 
$$
 \|T\|_{4m-4k+2}^{2m-2k+1} \|T\|_{2}^{4k-2m-1}\leq 
\begin{cases}
d^{2k-m-2}\|T\|_{k}^{2k} & \text{ \emph{if} }2(2m-2k+1)\leq k\\
d^{2k+\frac{1}{k}+(\frac{2}{k}-1)m - \frac{9}{2}}\|T\|_{k}^{2k} & o.w.
\end{cases}
$$
Among all the non-zero $(P,Q)$, consider the largest variance it can contribute. Given $d,n$, both $\frac{d^{2k-m-2}}{n^{2k-m}}$ and $\frac{d^{2k-\frac{1}{k}+(\frac{2}{k}-1)m - \frac{9}{2}}}{n^{2k-m}}$ are monotonic function in $m$ and hence the largest contribution of a single $(P,Q)$ can be at most $
2^{2k}\fourthm^k \max(\frac{d^{k-2}}{n^{k}}, \frac{d^{\frac{3}{4}k-\frac{3}{2}}}{n^{\frac{3}{4}k+\frac{1}{2}}}, \frac{d^{\frac{1}{2}-\frac{1}{k}}}{n})
$
Notice that when $n\le d$, we have $\frac{d^{\frac{3}{4}k-\frac{3}{2}}}{n^{\frac{3}{4}k+\frac{1}{2}}}\le \frac{d^{k-2}}{n^{k}}$. And whenever $n\ge d\ge 1$, we have $\frac{d^{\frac{3}{4}k-\frac{3}{2}}}{n^{\frac{3}{4}k+\frac{1}{2}}} \le \frac{d^{\frac{1}{2}-\frac{1}{k}}}{n}$. Hence the above quantity is equivalent to $2^{2k}\fourthm^k \max(\frac{d^{k-2}}{n^{k}},\frac{d^{\frac{1}{2}-\frac{1}{k}}}{n})
$. Combining with the  ``$k=1$'' case yields the following upperbound for a single pair, $(P,Q)$: $2^{2k} \fourthm^k\max( \frac{d^{k-2}}{n^k}, \frac{d^{\frac{1}{2}-\frac{1}{k}}}{n},\frac{1}{n})\|\T\|_{k}^{2k}$. The number of pairs $(P,Q)$, can be upper bounded by $(2k)^{2k}\cdot (4k)^{4k}$, since $P,Q$ are partitions of size $2k,4k$ respectively. Multiplying the number of pairs by the contribution of each of them, we have the following quantity:
$$
2^{12k}k^{6k}\fourthm^k \max( \frac{d^{k-2}}{n^k}, \frac{d^{\frac{1}{2}-\frac{1}{k}}}{n},\frac{1}{n})\|T\|_{k}^{2k}.
$$
Setting $f(k)=2^{12k}k^{6k}\fourthm^k$ yields the claimed proposition.
\end{proof}

\end{document}